%% file: main.tex
\documentclass[a4paper,11pt]{article}
\usepackage[margin=1in]{geometry}
\usepackage{import}
\usepackage{amsmath}
\usepackage{mathtools}
\usepackage{dsfont}
\usepackage{cite}
\usepackage[utf8]{inputenc}
\usepackage{csquotes}

\newcommand{\printfnsymbol}[1]{%
  \textsuperscript{\@fnsymbol{#1}}%
}
\usepackage{bbm}
\include{defpack2}

\begin{document}
\title{\bf On the Convergence of Differentially Private Federated Learning on Non-Lipschitz Objectives, and with Normalized Client Updates}
\date{}
\author[$\bm{\dagger}$]{Rudrajit Das}
\author[$\bm{\ddagger}$]{Abolfazl Hashemi}
\author[$\bm{\dagger}$]{Sujay Sanghavi} \author[$\bm{\dagger}$]{Inderjit S. Dhillon}
\affil[$\bm{\dagger}$]{University of Texas at Austin}
\affil[$\bm{\ddagger}$]{Purdue University}
\maketitle

\begin{abstract}
\vspace{0.1 cm}
\noindent There is a dearth of convergence results for differentially private federated learning (FL) with non-Lipschitz objective functions (i.e., when gradient norms are not bounded). The primary reason for this is that the clipping operation (i.e., projection onto an $\ell_2$ ball of a fixed radius called the clipping threshold) for bounding the sensitivity of the average update to each client's update introduces bias depending on the clipping threshold and the number of local steps in FL, and analyzing this is not easy. For Lipschitz functions, the Lipschitz constant serves as a trivial clipping threshold with zero bias. 
However, Lipschitzness does not hold in many practical settings; moreover, verifying it and computing the Lipschitz constant is hard. Thus, the choice of the clipping threshold is non-trivial and requires a lot of tuning in practice. 
In this paper, we provide the first convergence result for private FL 
on smooth \textit{convex} objectives \textit{for a general clipping threshold} -- \textit{without assuming Lipschitzness}. We also look at a simpler alternative to clipping (for bounding sensitivity) which is \textit{normalization} -- where we use only a scaled version of the unit vector along the client updates, completely discarding the magnitude information. {The resulting normalization-based private FL algorithm is theoretically shown to have better convergence than its clipping-based counterpart on smooth convex functions. 
We corroborate our theory with synthetic experiments as well as experiments on benchmarking datasets.}
\end{abstract}

\section{Introduction}
\label{sec:intro}
Collaborative machine learning (ML) schemes such as federated learning (FL) \cite{mcmahan2017communication} are growing at an unprecedented rate. In contrast to the conventional centralized paradigm of training, wherein all the data is stored in a central database, FL (and in general, a collaborative ML scheme) enables training ML models from \textit{decentralized} and \textit{heterogeneous} data through collaboration of many participants, e.g., mobile devices, each with different data and capabilities. 
In a standard FL setting, there are $n$ clients (e.g., mobile phones or sensors), each with their own decentralized data, and a central server that is trying to train a model, parameterized by $\bm{w} \in \mathbb{R}^d$, using the clients' data.
Suppose the $i^{\text{th}}$ client has $m$ training examples/samples \footnote{In general, each client may have different number of training examples. We consider the case of equal number of examples per client for ease of exposition.} $\{\bm{x}_1^{(i)},\ldots,\bm{x}_{m}^{(i)}\} := \mathcal{D}_i$, drawn from some distribution $\mathcal{P}_i$. Then the $i^{\text{th}}$ client has an objective function $f_i(\bm{w})$ which is the average loss, w.r.t. some loss function $\ell$, over its $m$ samples, and the central server tries to optimize the average \footnote{In general, this average is a weighted one with the weight of a client being proportional to the number of samples in that client.} loss $f(\bm{w})$, over the $n$ clients, i.e.,
\begin{equation}
    \label{eq:fl-def}
    f(\bm{w}) := \frac{1}{n} \sum_{i=1}^{n} {f_i}(\bm{w}), \text{ where } f_i(\bm{w}) := \frac{1}{m}\sum_{j=1}^{m} \ell(\bm{x}_j^{(i)},\bm{w}).
\end{equation}
The setting where the data distributions of all the clients are identical, i.e. $\mathcal{P}_1 = \ldots = \mathcal{P}_n$, is known as the \enquote{homogeneous} setting.
Other settings are known as \enquote{heterogeneous} settings. We quantify heterogeneity in more detail in \Cref{prelim} (see \Cref{def:het}).

The key algorithmic idea of FL is Federated Averaging commonly abbreviated as \texttt{FedAvg} \cite{mcmahan2017communication}.
In \texttt{FedAvg}, at every round, the server randomly chooses a subset of the clients and sends them the latest global model. These clients then undertake \textit{multiple} steps of local updates (on the global model received from the server) with their respective data based on (stochastic) gradient descent, and then communicate back their respective updated local models to the server. The server then averages the clients' local models to update the global model (hence the name \texttt{FedAvg}). \texttt{FedAvg} forms the basis of more advanced federated optimization algorithms. For the sake of completeness, we state \texttt{FedAvg} in \Cref{alg:fedavg} (\Cref{fedavg-sec}). The convergence of \texttt{FedAvg} as well as other FL algorithms depends heavily on the number of local updates as well as the degree of data heterogeneity -- specifically, for the same number of local updates, the convergence worsens as the amount of heterogeneity increases.

Despite the locality of data storage in FL, information-sharing opens the door to the possibility of sabotaging the security of personal data through communication. Hence, it is crucial to devise effective, privacy-preserving communication strategies that ensure the integrity and confidentiality of user data. Differential privacy (DP) \cite{dwork2006calibrating} is a popular privacy-quantifying framework that is being incorporated in the training of ML models. In particular, DP focuses on a learning algorithm's sensitivity to an individual's data; a less sensitive algorithm is less likely to leak individuals' private details through its output. This idea has laid the foundation for designing a simple strategy to ensure privacy by adding random Gaussian or Laplacian noise to the output, where the noise is scaled according to the algorithm's sensitivity to an individual's data. We talk about DP in more detail in \Cref{prelim}.

There has been a lot of work on differentially private optimization in order to enable private training of ML models. In this regard, DP-SGD \cite{abadi2016deep} is the most widely used private optimization algorithm in the centralized setting. 
It is essentially the same as regular SGD, except that Gaussian noise is added to the average of the \enquote{clipped} per-sample gradients (or updates) for privacy. 
There is a natural extension of DP-SGD to the federated setting based on \texttt{FedAvg}, wherein the server receives a noise-perturbed average of the \enquote{clipped} client updates \cite{geyer2017differentially,thakkar2019differentially}; this is called \texttt{DP-FedAvg} (with clipping) and it is stated in \Cref{alg:dp-fedavg}.
Specifically, if the original update is $\bm{u}$, then its clipped version is $\bm{u}\min(1, \frac{C}{\|\bm{u}\|_2})$, for some threshold $C$; notice that this is the projection of $\bm{u}$ onto an $\ell_2$ ball of radius $C$ centered at the origin. Clipping is performed to bound the sensitivity of the average update to each individual 
update, which is required to set the variance of the added Gaussian noise; specifically, the noise variance is proportional to $C^2$.

While the privacy aspect of DP-SGD and its variants, both in the centralized and federated setting, is typically the main consideration, the optimization aspect -- particularly due to clipping -- is not given that much attention. Specifically, the average of the clipped updates is biased and the amount of bias depends on the clipping threshold $C$ -- the higher the value of $C$, the lower is the bias, and vice-versa. But as mentioned before, the noise variance is proportional to $C^2$. Thus, the choice of the clipping threshold $C$ is associated with an intrinsic tension between the bias and variance of the noise-perturbed average of the clipped updates, which impacts the rate of convergence. 

To provide convergence guarantees for DP-SGD, most prior works assume that the per-sample losses are Lipschitz (i.e., they have bounded gradients); under this assumption, setting $C$ equal to the Lipschitz constant results in zero bias, making the convergence analysis trivial.  But in practice, we cannot ascertain the Lipschitzness property, let alone figuring out the Lipschitz constant, due to which the choice of the clipping threshold is not trivial and requires a lot of tuning. So ideally, we would like to have convergence results for non-Lipschitz functions. However, there aren't too many in the literature, primarily because analyzing the clipping bias is not easy, and more so in the federated setting due to \textit{multiple local updates}. A few works in the centralized setting do provide some results for the non-Lipschitz case by making more relaxed assumptions \cite{chen2020understanding,wang2020differentially,kamath2021improved,bu2021convergence}; we discuss these in \Cref{sec:rel_wrk}. However, in the more challenging federated setting with multiple local update steps, there is no result even for the convex non-Lipschitz case. In this work, we provide the \textit{first convergence result for differentially private federated convex optimization with a general clipping
threshold, while not assuming Lipschitzness} or making any other relaxed assumption; see Theorems \ref{thm-clip-cvx-short} and \ref{cor-1}. Moreover, prior works do not consider whether performing multiple local update steps is indeed beneficial (or not) for private optimization; we make the first attempt to analyze this theoretically. Informally, under an extra assumption, our result indicates that multiple local update steps are beneficial if the degree of heterogeneity of the data is dominated by poor choice of initialization (of the model parameters) for training. See (c) in the discussion after \Cref{thm-clip-cvx-short}.

Further, we also propose a simpler alternative (compared to clipping) for bounding the sensitivity which is to always \textit{normalize} the individual client updates; specifically, if the original update is $\bm{u}$, then its normalized version is $\bm{u} (\frac{C}{\|\bm{u}\|_2})$, for some appropriate scaling factor $C$. Surprisingly, this simpler option has not been considered by prior works on private optimization. The resultant private FL algorithm, where we replace clipping by normalization, is summarized in \Cref{alg:dp-fedavg-2} and we call it \texttt{DP-NormFedAvg}. 
We explain why/how/when the simpler alternative of normalization will offer better convergence than clipping in private optimization both theoretically (see \Cref{sec:comp-1}) as well as intuitively (see \Cref{sec:dp-normfed}); we also elaborate on this while summarizing our contributions next. 
\\
\\
\noindent Our main \textbf{contributions} are summarized next:
\\
\noindent \textbf{(a)} In \Cref{thm-clip-cvx-short}, we provide a convergence result for \texttt{DP-FedAvg} with clipping (\Cref{alg:dp-fedavg}) which is the \textit{first convergence result for differentially private federated convex optimization with a general clipping threshold and without assuming Lipschitzness}, followed by a simplified (but less tight) convergence result in \Cref{cor-1}. Based on our derived result, we also attempt to quantify the benefit/harm of performing multiple local update steps in private optimization. Informally, under an extra assumption (\Cref{local_steps_asmp}), we show that \textit{multiple local updates are beneficial if the effect of poor initialization (of the model parameters) outweighs the effect of data heterogeneity by a factor depending on the privacy level}; see (c) in the discussion after \Cref{thm-clip-cvx-short}. 
\\
\\
\noindent \textbf{(b)} In \Cref{norm-sec}, we present \texttt{DP-NormFedAvg} (\Cref{alg:dp-fedavg-2}) where we replace update clipping by the simpler alternative of update \textit{normalization} (i.e., sending a scaled version of the \textit{unit vector} along the update) for bounding the sensitivity. We provide a convergence result for \texttt{DP-NormFedAvg} in \Cref{thm-new-clip-cvx-short} and compare it against the result of \texttt{DP-FedAvg} with clipping (i.e., \Cref{thm-clip-cvx-short}), showing that when the effect of poor initialization of the model parameters is more severe than the degree of data heterogeneity and/or if we can train for a large number of rounds, we expect the simpler alternative of normalization to offer better convergence than clipping in private optimization; see \Cref{rmk-cmp} and \Cref{sec:comp-1} for details. Intuitively, this happens because \textit{normalization has a higher signal (i.e., update norm) to noise ratio than clipping}; this aspect is discussed in detail in \Cref{sec:dp-normfed}.
\\
\\
\noindent \textbf{(c)} We demonstrate the superiority of normalization over clipping via experiments on a synthetic quadratic problem in \Cref{sec:norm-vs-clip-expt} as well as on three benchmarking datasets, viz., Fashion MNIST \cite{xiao2017fashion}, CIFAR-10 and CIFAR-100 in \Cref{sec:expts}. For our synthetic experiment, we show that normalization has a higher signal to noise ratio than clipping (as mentioned above) in \Cref{fig:2}, and that the trajectory of normalization (projected in 2D space) reaches closer to the optimum of the function than the trajectory of clipping in \Cref{fig:traj}. In the experiments on benchmarking datasets, for $\varepsilon = 5$, the improvement offered by normalization over clipping w.r.t. the test accuracy is more than $2.8$\% for CIFAR-100, $2.1$\% for Fashion MNIST and $1.5$\% for CIFAR-10; see \Cref{tab1}.

\section{Preliminaries}
\label{prelim}
In this work, we are able to naturally quantify the effect of heterogeneity on convergence as follows.
\begin{definition}[\textbf{Heterogeneity}]
\label{def:het}
Let $\bm{w}^{*} \in \arg \min_{\bm{w}' \in \mathbb{R}^d} f(\bm{w}')$ and $\Delta_i^{*} := f_i(\bm{w}^{*}) - \min_{\bm{w}' \in \mathbb{R}^d} f_i(\bm{w}') \geq 0$. 
Then the heterogeneity of the system is quantified by some increasing function of the $\Delta_i^{*}$'s.
\end{definition}
\noindent The above way of quantifying heterogeneity shows up naturally in our convergence results for private FL assuming that the $f_i$'s are convex and smooth. The exact function of $\Delta_i^{*}$'s (quantifying heterogeneity) depends on the algorithm as well as data, and this will become clear when we present the convergence results. Also note that if the per-client data distributions (i.e, $\mathcal{P}_i$'s) are similar, then we expect the $\Delta_i^{*}$'s to be small indicating smaller heterogeneity.
\\
\\
\noindent \textbf{Differential Privacy (DP):} Suppose we have a collection of datasets ${D}_c$ and a query function $h:{D}_c \xrightarrow{} \mathcal{X}$. Two datasets $\mathcal{D} \in {D}_c$ and $\mathcal{D}' \in {D}_c$ are said to be neighboring if they differ in exactly one sample, and we denote this by $|\mathcal{D} - \mathcal{D}'| = 1$. A randomized mechanism $\mathcal{M}:\mathcal{X} \xrightarrow{} \mathcal{Y}$ is said to be $(\varepsilon,\delta)$-DP, if for any two neighboring datasets $\mathcal{D},\mathcal{D}' \in {D}_c$ and for any measurable subset of outputs $\mathcal{R} \in \mathcal{Y}$,
\begin{equation}
    \label{eq:dp-1}
    \mathbb{P}(\mathcal{M}(h(\mathcal{D})) \in \mathcal{R}) \leq e^{\varepsilon} \mathbb{P}(\mathcal{M}(h(\mathcal{D}')) \in \mathcal{R}) + \delta.
\end{equation}
When $\delta = 0$, it is commonly known as pure DP. Otherwise, it is known as approximate DP.

Adding random Gaussian noise to the output of $h(.)$ above is the customary approach to provide DP; this is known as the Gaussian mechanism and we formally define it below.
\begin{definition}[\textbf{Gaussian mechanism \cite{dwork2014algorithmic}}]
Suppose $\mathcal{X}$ (i.e., the range of the query function $h$ above) is $\mathbb{R}^p$. Let $\Delta_2 := \sup_{\mathcal{D}, \mathcal{D}' \in {D}_c: |\mathcal{D} - \mathcal{D}'| = 1}\|h(\mathcal{D}) - h(\mathcal{D}')\|_2$. If we set
\begin{equation*}
    \mathcal{M}(h(\mathcal{D})) = h(\mathcal{D}) + \bm{Z},
\end{equation*}
where $\bm{Z} \sim \mathcal{N}\Big(\vec{0}_p, \frac{2 \log(1.25/\delta) \Delta_2^2}{\varepsilon^2} \textup{I}_p\Big)$, 
then the mechanism $\mathcal{M}$ is $(\varepsilon,\delta)$-DP. 
\end{definition}
\noindent The Gaussian mechanism is also employed in private optimization \cite{abadi2016deep}.

\begin{definition}[\textbf{Lipschitz}]
\label{def-lip}
A function $g:\Theta \xrightarrow{} \mathbb{R}$ is to said to be $G$-Lipschitz if $\sup_{\bm{\theta} \in \Theta}\|\nabla g(\bm{\theta})\|_2 \leq G$.
\end{definition}

\begin{definition}[\textbf{Smoothness}]
A function $g:\Theta \xrightarrow{} \mathbb{R}$ is to said to be $L$-smooth if for all $\bm{\theta}, \bm{\theta}' \in \Theta$, $\|\nabla g(\bm{\theta}) - \nabla g(\bm{\theta}')\|_2 \leq L\|\bm{\theta} - \bm{\theta}'\|_2$. If $g$ is twice differentiable, then for all $\bm{\theta}, \bm{\theta}' \in \Theta$:
\begin{equation*}
    g(\bm{\theta}') \leq g(\bm{\theta}) + \langle \nabla g(\bm{\theta}), \bm{\theta}' - \bm{\theta} \rangle + \frac{L}{2}\|\bm{\theta}' - \bm{\theta}\|_2^2.
\end{equation*}
\end{definition}
\begin{definition}[\textbf{A Key Quantity}]
\label{key_qty}
All the theoretical results in this paper are expressed in terms of the following key quantity:
\begin{equation}
    \rho := \frac{\sqrt{q d \log({1}/{\delta})}}{n\varepsilon},
\end{equation}
where $(\varepsilon,\delta)$-DP is the desired privacy level, $n$ is the number of samples, $d$ is the parameter dimension and $q$ is the absolute constant in \Cref{thm-dp}. Further, all our results are for the non-vacuous privacy regime, i.e., when $\varepsilon$ is finite and $\delta < 1$, where $\rho > 0$. 
Finally, we also assume that $n$ is sufficiently large so that $\rho < 1$.
\end{definition}
\noindent Note that $\rho$ increases as the level of privacy increases (i.e., $\varepsilon$ and $\delta$ decrease), and vice versa.
\\
\\
\noindent \textbf{Notation:} Throughout the rest of this paper, we denote the $\ell_2$ norm simply by $\|.\|$ (omitting the subscript 2). Vectors and matrices are written in boldface. 
We denote the uniform distribution over the integers $\{0,\ldots,a\}$ (where $a \in \mathbb{N}$) by $\text{unif}[0,a]$.
The function $\text{clip}: \mathbb{R}^d \times \mathbb{R}^{+} \xrightarrow{} \mathbb{R}^d$ is defined as:
\begin{equation}
    \label{eq:clip-def}
    \text{clip}(\bm{z},c) := \bm{z}\min\Big(1, \frac{c}{\|\bm{z}\|}\Big). 
\end{equation}
$K$ is the number of communication rounds or the number of global updates, $E$ is the number of local updates per round, and $r$ is the number of clients that the server accesses in each round.
\\
\\
{The proofs of all theoretical results are in the Appendix.}


\section{Related Work}
\label{sec:rel_wrk}
\textbf{Differentially private optimization:} Most differentially private optimization algorithms for training ML models (both in the centralized and federated settings) are based off of DP-SGD, wherein the optimizer receives a Gaussian noise-perturbed average of the \textit{clipped} per-sample gradients (to guarantee DP), and the moments accountant method \cite{abadi2016deep}. 
Similar to and/or related to DP-SGD, there are several papers on private optimization algorithms in the {centralized} setting \cite{chaudhuri2008privacy, chaudhuri2011differentially,kifer2012private,song2013stochastic,duchi2013local,bassily2014private,talwar2015nearly,wu2017bolt,zhang2017efficient,wang2018differentially,iyengar2019towards,bassily2019private,feldman2020private,asi2021private} as well as in the federated and distributed (without multiple local updates) setting \cite{geyer2017differentially,agarwal2018cpsgd,thakkar2019differentially,li2019privacy,peterson2019private,nguyen2021flguard,girgis2021shuffled}. \texttt{DP-FedAvg} with clipping \cite{geyer2017differentially,thakkar2019differentially} (stated in \Cref{alg:dp-fedavg}) is the most standard private algorithm in the federated setting. Among these previously mentioned works in the centralized setting, the ones that do provide convergence guarantees assume Lipschitzness and they set the clipping threshold equal to the Lipschitz constant, obtaining a suboptimality gap (i.e., $\mathbb{E}[f(\bm{w}_\text{priv})] - \min_{\bm{w}} f(\bm{w})$, where $\bm{w}_\text{priv}$ is the output) in the convex case of $\mathcal{O}(\rho)$, where $\rho = \mathcal{O}\Big(\frac{\sqrt{d \log({1}/{\delta})}}{n\varepsilon}\Big)$ is the key quantity defined in \Cref{key_qty}.
In fact, \cite{bassily2014private} show that in the convex Lipschitz case, the $\mathcal{O}(\rho)$ suboptimality gap is tight. 
However, as mentioned in \Cref{sec:intro}, Lipschitzness is not a very practical assumption, due to which it is important to obtain convergence guarantees under weaker assumptions where there is no trivial clipping threshold. To that end, there a few results in the \textit{centralized} setting that do not make the simplistic Lipschitzness assumption, but instead make more relaxed assumptions such as gradients having bounded moments \cite{wang2020differentially,kamath2021improved} or the stochastic gradient noise having a symmetric probability distribution function \cite{chen2020understanding}. Also, \cite{bu2021convergence} analyze full-batch DP-GD from the NTK perspective for deep learning models. In comparison, there are hardly any convergence results for \textit{private federated optimization} (which is harder to analyze due to multiple local steps) of non-Lipschitz objectives; \cite{zhang2021understanding} provide a complicated result for the nonconvex case, but surprisingly there is no result for the convex case. In addition, the role of multiple local steps in \textit{private} FL has not been theoretically studied.
\\
\\
\noindent \textbf{Normalized gradient descent (GD) and related methods:} In the centralized setting, \cite{hazan2015beyond} propose (Stochastic) Normalized GD. This is based on a similar idea as \texttt{DP-NormFedAvg} -- instead of using the (stochastic) gradient, use the \textit{unit vector} along the (stochastic) gradient for the update. Extensions of this method incorporating momentum \cite{you2017large,you2019large,cutkosky2020momentum} have been shown to significantly improve the training time of very large models such as BERT in the centralized setting. 
In the FL setting, \cite{charles2021large} propose Normalized \texttt{FedAvg}, where the server uses a normalized version of the average of client updates (and \textit{not} the average of normalized client updates, which is what we do) to improve training.
However, it must be noted here that these works perform (some kind of) normalization to accelerate \textit{non-private} training, whereas we are proposing normalization as an alternative \textit{sensitivity bounding} mechanism to improve \textit{private} training compared to the usual mechanism of clipping. 

\section{Convergence of Vanilla \texttt{DP-FedAvg} with Client-Update Clipping}
First, we focus on the most standard version of \texttt{DP-FedAvg} involving client-update \textit{clipping}, which is summarized in \Cref{alg:dp-fedavg}.
The primary difference from \texttt{FedAvg} is in {lines 9, 10 and 12} of \Cref{alg:dp-fedavg}.
Each client in the selected subset of clients sends its \textit{clipped} update plus zero-mean Gaussian noise (for differential privacy) to the server; since Gaussian noise is additive, we can add it at the clients itself.
The server then computes the mean of the noisy \textit{clipped} client updates that it received (i.e., $\bm{a}_k$) and then uses it to update the global model similar to \texttt{FedAvg}, except with a potentially different global learning rate ($\beta_k$) than the local learning rate ($\eta_k$). Since each $\bm{\zeta}^{(i)}_k$ (i.e., noise added at client $i$) is $\mathcal{N}(\vec{0}_d, r \sigma^2 \bm{\textup{I}}_d)$, the average noise at the server is $\mathcal{N}(\vec{0}_d, \sigma^2 \bm{\textup{I}}_d)$.
Using the moments accountant method of \cite{abadi2016deep}, we now specify the value of $\sigma^2$ required to make \Cref{alg:dp-fedavg} $(\varepsilon, \delta)$-DP.
\begin{theorem}[\textbf{\cite{abadi2016deep}}]
\label{thm-dp}
For any $0<\varepsilon < \mathcal{O}\Big(\frac{r^2 K}{n^2}\Big)$, \Cref{alg:dp-fedavg} will be $(\varepsilon, \delta)$-DP as long as
\begin{equation}
    \label{eq:dp-noise}
    \sigma^2 = q K C^2 \frac{\log({1}/{\delta})}{n^2 \varepsilon^2},
\end{equation}
where $q>0$ is an absolute constant.
\end{theorem}
Note that the original DP-SGD algorithm of \cite{abadi2016deep} returns the last iterate (i.e., $\bm{w}_K$) as the output, and Theorem 1 in their paper guarantees that the last iterate is $(\varepsilon, \delta)$-DP by setting $\sigma^2$ as per \cref{eq:dp-noise}. But if the last iterate is $(\varepsilon, \delta)$-DP, then so is \textit{any} other iterate (due to additivity of the privacy cost), from which \Cref{thm-dp} follows.

\begin{algorithm}
\caption{\texttt{DP-FedAvg} (with clipping)}
\label{alg:dp-fedavg}
\begin{algorithmic}[1]
\STATE {\bfseries Input:}
Initial point $\bm{w}_0$, number of rounds of communication $K$, number of local updates per round $E$, local learning rates $\{\eta_k\}_{k=0}^{K-1}$, global learning rates $\{\beta_k\}_{k=0}^{K-1}$, clipping threshold $C$, number of participating clients in each round $r$ and noise variance $\sigma^2$.
\vspace{0.1 cm}
\FOR{$k=0,\dots,K-1$}
\vspace{0.1 cm}
\STATE 
Server sends $\bm{w}_k$ to a random set $\mathcal{S}_k$ of clients, formed by sampling each client $\in [n]$ with probability $r/n$.
\vspace{0.1 cm}
\FOR{client $i \in \mathcal{S}_k$}
\vspace{0.1 cm}
\STATE Set $\bm{w}_{k,0}^{(i)} = \bm{w}_k$.
\vspace{0.1 cm}
\FOR{$\tau = 0,\dots,E-1$}
\vspace{0.1 cm}
\STATE Update $\bm{w}_{k,\tau+1}^{(i)} \xleftarrow{} \bm{w}_{k,\tau}^{(i)} - \eta_k \nabla f_i(\bm{w}_{k,\tau}^{(i)})$. 
\vspace{0.1 cm}
\ENDFOR
\vspace{0.1 cm}
\STATE Let $\bm{u}^{(i)}_k = \frac{\bm{w}_{k} - \bm{w}_{k,E}^{(i)}}{\eta_k}$ and $\bm{g}^{(i)}_k = \text{clip}\big(\bm{u}^{(i)}_k, C \big) = \bm{u}^{(i)}_k \min\Big(1, \frac{C}{\|\bm{u}^{(i)}_k\|}\Big)$. // \texttt{($\bm{u}^{(i)}_k$ is client $i$'s update.)}
\vspace{0.1 cm}
\STATE Send $(\bm{g}^{(i)}_k + \bm{\zeta}^{(i)}_k)$ to the server, where $\bm{\zeta}^{(i)}_k \sim \mathcal{N}(\vec{0}_d, r \sigma^2 \bm{\textup{I}}_d)$.
\vspace{0.1 cm}
\ENDFOR
\vspace{0.1 cm}
\STATE Update $\bm{w}_{k+1} \xleftarrow{} \bm{w}_k - \beta_k \bm{a}_k$, where $\bm{a}_k = \frac{1}{r}\sum_{i \in \mathcal{S}_k}(\bm{g}^{(i)}_k + \bm{\zeta}^{(i)}_k)$.
\vspace{0.1 cm}
\STATE Return $\bm{w}_\text{priv} = \bm{w}_{\tilde{k}}$, where $\tilde{k} \sim \text{unif}[0,K-1]$.
\vspace{0.1 cm}
\ENDFOR
\end{algorithmic}
\end{algorithm}
We now present the abridged convergence result for \Cref{alg:dp-fedavg}; the full version and proof can be found in \Cref{thm1-details}. 
\begin{theorem}[\textbf{Convergence of \texttt{DP-FedAvg} with Clipping: Convex Case}]
\label{thm-clip-cvx-short}
Suppose each $f_i$ is convex and $L$-smooth over $\mathbb{R}^d$. Let $\hat{C} := \frac{C}{E}$, where $C$ is the clipping threshold used in \Cref{alg:dp-fedavg}. For any $\bm{w}^{*} \in \arg \min_{\bm{w}' \in \mathbb{R}^d} f(\bm{w}')$ and $\Delta_i^{*} := f_i(\bm{w}^{*}) - \min_{\bm{w}' \in \mathbb{R}^d} f_i(\bm{w}') \geq 0$, \Cref{alg:dp-fedavg} with $\hat{C} \geq 4 \sqrt{L \max_{j \in [n]} \Delta_j^{*}}$, $\beta_k = \eta_k = \eta = \frac{\rho}{2 \alpha L}$, $K = \big(\frac{2 \alpha \gamma}{\hat{C} E}\big) \frac{1}{\rho^2}$ and $E \leq \frac{\alpha}{2 \rho}$, where $\gamma > 0$ and $\alpha \geq 1$ are constants of our choice, has the following convergence guarantee:
\begin{multline}
        \label{eq:thm1-1}
        \mathbb{E}\Bigg[\frac{1}{n}\sum_{i=1}^n \Bigg(\mathbbm{1}(\|\bm{u}_{\tilde{k}}^{(i)}\| \leq  \hat{C}E) \Big(2 - \frac{\rho E}{\alpha} - \frac{\rho^2 E^2}{\alpha^2}\Big)(f_i(\bm{w}_{\tilde{k}}) - f_i(\bm{w}^{*})) +  \mathbbm{1}(\|\bm{u}_{\tilde{k}}^{(i)}\| > \hat{C}E) \Big(\frac{3 \hat{C}}{8 L E} \|\bm{u}_{\tilde{k}}^{(i)}\|\Big)\Bigg)\Bigg]
        \\
        \leq \underbrace{\hat{C}\Big(\frac{L \|\bm{w}_{0} - \bm{w}^{*}\|^2}{\gamma} + \frac{\gamma}{L}\Big)\rho}_{:= \textup{A(effect of \textbf{initialization})}} + \underbrace{\Big(\frac{3 E}{2 \alpha}\Big) \mathbb{E}\Bigg[\frac{1}{n}\sum_{i=1}^n {\mathbbm{1}(\|\bm{u}_{\tilde{k}}^{(i)}\| \leq  \hat{C}E)} \Delta_i^{*}\Bigg] \rho}_{:= \textup{B(effect of \textbf{heterogeneity})}},
\end{multline}
with $\tilde{k} \sim \textup{unif}[0,K-1]$.
\end{theorem}
In the above result, we remind the reader that $\bm{u}_{\tilde{k}}^{(i)}$ is the $i^{\text{th}}$ client's update at a random round number $\tilde{k}$ (see line 9 in \Cref{alg:dp-fedavg}). Also, this result is for the non-vacuous privacy regime, where $\rho := \frac{\sqrt{q d  \log({1}/{\delta})}}{n\varepsilon} > 0$\footnote{If $\rho = 0$, \Cref{alg:dp-fedavg} reduces to vanilla \textit{non-private} \texttt{FedAvg}. One can recover the convergence result of vanilla \texttt{FedAvg} by just changing the learning rates appropriately in the proofs of our theorems.}.

We also present a simplified, but less tight, convergence result based on \Cref{thm-clip-cvx-short}; its proof is in \Cref{cor-1-pf}.
\begin{theorem}[\textbf{Simplified version of \Cref{thm-clip-cvx-short}}]
\label{cor-1}
With $\gamma = \mathcal{O}(L \|\bm{w}_{0} - \bm{w}^{*}\|)$ and $\alpha = \mathcal{O}(1)$, the convergence result of \Cref{thm-clip-cvx-short} can be simplified to:
\begin{equation}
     \mathbb{E}\Bigg[\frac{1}{n}\sum_{i=1}^n \min\Bigg(f_i(\bm{w}_{\tilde{k}}) - f_i(\bm{w}^{*}), \mathcal{O}\Big(\frac{\hat{C}}{L} \|\nabla f_i(\bm{w}_{\tilde{k}})\|\Big)\Bigg)\Bigg]
     \leq \mathcal{O}\Bigg(\hat{C} \|\bm{w}_{0} - \bm{w}^{*}\| + E \Big(\frac{1}{n}\sum_{i=1}^n \Delta_i^{*}\Big)\Bigg)\rho.
\end{equation}
\end{theorem}

\noindent We now delineate the key implications of \Cref{thm-clip-cvx-short}.
\\
\\
\noindent \textbf{(a) Convergence without assuming Lipschitzness:} Note that \Cref{thm-clip-cvx-short} does not assume any $f_i$ to be Lipschitz. To our knowledge, this is the \textit{first convergence result for private federated convex optimization with a general clipping threshold, and without assuming Lipschitzness}.

Let us also see what happens in the Lipschitz case. For that, suppose each $f_i$ is $G$-Lipschitz over $\mathbb{R}^d$. So if we set $\hat{C} = G$, then $\|\bm{u}_k^{(i)}\| \leq \hat{C} E$ for all $k$. Now, using the fact that $(2 - \frac{\rho E}{\alpha} - \frac{\rho^2 E^2}{\alpha^2}) \geq \frac{5}{4}$ for $E \leq \frac{\alpha}{2 \rho}$, the convergence result in \Cref{thm-clip-cvx-short} (i.e., \cref{eq:thm1-1}) reduces to:
\begin{equation}
        \label{eq:thm1-2}
        \mathbb{E}[f(\bm{w}_{\tilde{k}})] - f(\bm{w}^{*}) \leq \frac{4}{5}\Bigg(G\Big(\frac{L \|\bm{w}_{0} - \bm{w}^{*}\|^2}{\gamma} + \frac{\gamma}{L}\Big) + \frac{3 E}{2\alpha} \Big(\frac{1}{n}\sum_{i=1}^n \Delta_i^{*}\Big)\Bigg)\rho.
\end{equation}
Thus with $E = \mathcal{O}(1)$, our bound matches the lower bound of \cite{bassily2014private} for the centralized convex and Lipschitz case with respect to the dependence on $\rho$. 
\\
\\
\noindent \textbf{(b) Effect of initialization and heterogeneity:}
Observe that our convergence result depends on two things: (i) term $\textup{A}$ in \cref{eq:thm1-1}, i.e. the distance of the \textbf{initialization} $\bm{w}_0$ from the optimum $\bm{w}^{*}$, and (ii) term $\textup{B}$ in \cref{eq:thm1-1}, i.e. the degree of \textbf{heterogeneity} which itself depends on the $\Delta_i^{*}$'s (as per \Cref{def:het}). A high (respectively, low) degree of heterogeneity implies high (respectively, low) values of $\Delta_i^{*}$'s, which leads to worse (respectively, better) convergence. Also, as we increase $\alpha$ in \Cref{thm-clip-cvx-short}, i.e. increase the number of rounds $K$, the effect of the heterogeneity term (B) dies down. However, the effect of the initialization term (A) cannot be diminished by increasing $\alpha$.
\\
\\
{\noindent \textbf{(c) Effect of multiple local steps:} 
Characterizing whether having multiple local steps, i.e. $E > 1$, is beneficial or detrimental is not obvious from \Cref{thm-clip-cvx-short}. The RHS of \cref{eq:thm1-1} seems to suggest that the convergence result gets worse as we increase $E$ -- but this is whilst keeping $\hat{C}$ fixed. The intricacy here is that for the \enquote{same amount of clipping}, the required value of $\hat{C}$ is a non-increasing function of $E$.
To make this more precise, let us consider two values of $E$, say $E_1$ and $E_2$ where $E_1 < E_2$. Suppose the corresponding clipping thresholds that we use are $C_1 = \hat{C}_1 E_1$ and $C_2 = \hat{C}_2 E_2$, respectively.
Now if we wish to have $\mathbbm{1}(\|\bm{u}_{\tilde{k}}^{(i)}\| \leq \hat{C}_1 E_1) = \mathbbm{1}(\|\bm{u}_{\tilde{k}}^{(i)}\| \leq \hat{C}_2 E_2)$, i.e. the \enquote{same amount of clipping} with $E_1$ and $E_2$, then $\hat{C_2} \leq \hat{C_1}$. This is because we are doing gradient descent on convex functions locally, due to which $\frac{\big\|\bm{u}_{\tilde{k}}^{(i)}\big\|}{E} = \frac{\big\|\sum_{\tau=0}^{E-1}\nabla f_i(\bm{w}_{\tilde{k}, \tau}^{(i)})\big\|}{E}$ is a non-increasing function of $E$. 
However, quantifying the extent of \enquote{non-increasingness} of ${\|\bm{u}_{\tilde{k}}^{(i)}\|}/{E}$ -- which allows us to provide choices of $\hat{C}$ as a function of $E$ -- is not easy (and perhaps not possible) without making more assumptions other than convexity and smoothness. So, we now make a couple of extra assumptions (one of which is the standard Lipschitzness assumption) in \Cref{local_steps_asmp}, which then allows us to illustrate and quantify the \enquote{non-increasingness} of $\hat{C}$ as a function of $E$ in \Cref{local_steps}.
\begin{assumption}
\label{local_steps_asmp}
\textbf{\textup{(i)}} For any $\bm{w} \in \mathbb{R}^d$ and each $i \in [n]$, we have that:
\begin{equation}
    \label{asmp-extra}
    \|\nabla f_i(\bm{w} - \eta \nabla f_i(\bm{w})) - \nabla f_i(\bm{w})\| \geq \eta \lambda \|\nabla f_i(\bm{w})\|,
\end{equation}
for some $0 < \lambda \leq L$ and $\eta \leq \frac{\rho}{2L}$. 
\\
\\
\textbf{\textup{(ii)}} Additionally, each $f_i$ is $G$-Lipschitz over $\mathbb{R}^d$. 
\end{assumption}
{\noindent \Cref{local_steps_asmp} (i) can be also interpreted as a lower bound on the norm of the product of the Hessian matrix and the gradient vector. This is because for small enough $\eta$, we have:
\begin{equation}
    \|\nabla f_i(\bm{w} - \eta \nabla f_i(\bm{w})) - \nabla f_i(\bm{w})\| = \Theta(\eta \|\nabla^2 f_i(\bm{w}) \nabla f_i(\bm{w})\|).
\end{equation}
So basically for \Cref{local_steps_asmp} (i) to hold, we are assuming $\|\nabla^2 f_i(\bm{w}) \nabla f_i(\bm{w})\| \geq \Omega(\lambda \|\nabla f_i(\bm{w})\|)$; note that a similar assumption has been made in \cite{daneshmand2018escaping}. Also, \Cref{local_steps_asmp} (i) is a weaker assumption than strong convexity. This is because strong convexity would imply that $\|\nabla^2 f_i(\bm{w}) \bm{v}\| \geq \mu \|\bm{v}\|$ for \textit{any} $\bm{v} \in \mathbb{R}^d$ and some $\mu > 0$, while we assume this to hold only for $\bm{v} = \nabla f_i(\bm{w})$ (with $\mu = \Omega(\lambda)$). 
}
\rudrajit{Can we improve the convergence of DP-FedAvg (i.e., either the dependence on $\rho$ or $K/E$) with this assumption? -- Probably NOT}
\begin{proposition} [\textbf{$\hat{C}$ under \Cref{local_steps_asmp}}]
\label{local_steps}
Set $\alpha = 1$ in \Cref{thm-clip-cvx-short}; this imposes the constraint $E \leq \frac{1}{2 \rho}$. Then, under \Cref{local_steps_asmp}, choosing $\hat{C} = G \Big(1 - \frac{11(E-1){\rho}}{64}\big(\frac{\lambda^2}{L^2}\big)\Big)$ in \Cref{thm-clip-cvx-short} ensures that $\mathbbm{1}(\|\bm{u}_{{k}}^{(i)}\| > \hat{C} E) = 0$ $\forall$ $k$, i.e. no clipping happens for all $E$.
\end{proposition}
\noindent The proof of \Cref{local_steps} is in \Cref{local_steps_pf}. Notice that $\hat{C}$, which is set so that the same (= zero) amount of clipping happens for all $E$ ($\leq \frac{1}{2\rho}$), is a non-increasing (more specifically, a decreasing) function of $E$ as mentioned before. It is worth mentioning here that the value of $\hat{C}$ in \Cref{local_steps} is not the tightest possible value, but even for the tightest value, the non-increasingness will hold.


To summarize, there is a tradeoff involved as far as the number of local steps $E$ is concerned. Increasing $E$ allows us to reduce $\hat{C}$ which mitigates the effect of initialization, i.e. term A in \cref{eq:thm1-1}, but at the cost of increasing the effect of heterogeneity, i.e. term B in \cref{eq:thm1-1}. 
To illustrate this tradeoff, let us plug in our choice of $\hat{C}$ derived in \Cref{local_steps} in the convergence result of \Cref{thm-clip-cvx-short} with $\gamma = L \|\bm{w}_{0} - \bm{w}^{*}\|$ (this choice is just for simplicity). After a bit of simplification, this yields:
\begin{equation}
        \label{eq:eg}
        \mathbb{E}[f(\bm{w}_{\tilde{k}})] - f(\bm{w}^{*}) \leq
        \Bigg(2 G {\|\bm{w}_{0} - \bm{w}^{*}\|} + \frac{6 E}{5} \Bigg\{\frac{1}{n}\sum_{i=1}^n \Delta_i^{*} - \frac{11 G \|\bm{w}_{0} - \bm{w}^{*}\| \rho}{48}\Big(\frac{\lambda^2}{L^2}\Big)\Bigg\}\Bigg)\rho.
\end{equation}
\Cref{eq:eg} tells us that if $\frac{1}{n}\sum_{i=1}^n \Delta_i^{*} < \mathcal{O}\Big({G} \big(\frac{\lambda^2}{L^2}\big)\Big) \|\bm{w}_{0} - \bm{w}^{*}\| \rho $, which in plain English basically means that if the degree of heterogeneity is less than the product of the distance of the initialization from the optimum and $\rho = \mathcal{O}\Big(\frac{\sqrt{d \log(1/\delta)}}{n \varepsilon}\Big)$ (and some other data-dependent constants), then having a large value of $E$ is beneficial; in particular, setting the maximum permissible value of $E$, which is $\frac{1}{2 \rho}$, is the best (in terms of smallest suboptimality gap). Otherwise, having a small value of $E$ is better; specifically, setting $E=1$ is the best. From \Cref{thm-clip-cvx-short}, recall that for $\alpha = 1$, $K = \Big(\frac{2 \gamma}{\hat{C} E}\Big) \frac{1}{\rho^2}$; so the higher we set $E$, the fewer the number of communication rounds needed.

From the above discussion, one should not form the opinion that a poor initialization, i.e., a $\bm{w}_0$ such that $\|\bm{w}_{0} - \bm{w}^{*}\|$ is large, is advantageous in private FL. This is because choosing such a $\bm{w}_0$ will increase the first term within the big round brackets in \cref{eq:eg}, i.e. $2 G {\|\bm{w}_{0} - \bm{w}^{*}\|}$, which happens to be the \textit{dominant} term -- leading to a high suboptimality gap by default.

}
\rudrajit{\textbf{MORE CHANGES NEEDED ABOVE?} If initialization is poor, anyways final error is gonna be bad...probably mention this.
}

\section{\texttt{DP-NormFedAvg}: \texttt{DP-FedAvg} with Client-Update Normalization (instead of Clipping)}
\label{norm-sec}
We define the normalization function $\text{norm}: \mathbb{R}^d - \{\vec{0}_d\} \times \mathbb{R}^{+} \xrightarrow{} \mathbb{R}^d$ as:
\begin{equation}
    \text{norm}(\bm{z},c) := \frac{c \bm{z}}{\|\bm{z}\|},
\end{equation}
where $c$ is the scaling factor. The parameter $c$ is analogous to the clipping threshold in the $\text{clip}(.)$ function. Also, note that $\|\text{norm}(\bm{z},c)\| \leq c$ holds. 
\\
\\
\noindent Here we propose to \textit{normalize} client-updates instead of clipping them, i.e., we propose to change {line 9} of \Cref{alg:dp-fedavg} as follows:
\begin{equation}
    \label{eq:new-clip}
    \bm{g}^{(i)}_k = 
    \text{norm}(\bm{u}_k^{(i)}, C).
\end{equation}
We call the resultant algorithm \texttt{DP-NormFedAvg} because it involves \textit{norm}alizing updates in \texttt{DP-FedAvg}. For completeness, we state it in \Cref{alg:dp-fedavg-2}; note the normalization step in line 9.
\begin{algorithm}
\caption{\texttt{DP-NormFedAvg}}
\label{alg:dp-fedavg-2}
\begin{algorithmic}[1]
\STATE {\bfseries Input:}
Initial point $\bm{w}_0$, number of rounds of communication $K$, number of local updates per round $E$, local learning rates $\{\eta_k\}_{k=0}^{K-1}$, global learning rates $\{\beta_k\}_{k=0}^{K-1}$, scaling factor $C$, number of participating clients in each round $r$ and noise variance $\sigma^2$.
\vspace{0.1 cm}
\FOR{$k=0,\dots,K-1$}
\vspace{0.1 cm}
\STATE 
Server sends $\bm{w}_k$ to a random set $\mathcal{S}_k$ of clients, formed by sampling each client $\in [n]$ with probability $r/n$.
\vspace{0.1 cm}
\FOR{client $i \in \mathcal{S}_k$}
\vspace{0.1 cm}
\STATE Set $\bm{w}_{k,0}^{(i)} = \bm{w}_k$.
\vspace{0.1 cm}
\FOR{$\tau = 0,\dots,E-1$}
\vspace{0.1 cm}
\STATE Update $\bm{w}_{k,\tau+1}^{(i)} \xleftarrow{} \bm{w}_{k,\tau}^{(i)} - \eta_k \nabla f_i(\bm{w}_{k,\tau}^{(i)})$.
\vspace{0.1 cm}
\ENDFOR
\vspace{0.1 cm}
\STATE Let $\bm{u}^{(i)}_k = \frac{\bm{w}_{k} - \bm{w}_{k,E}^{(i)}}{\eta_k}$ and {\color{blue}$\bm{g}^{(i)}_k = \text{norm}\big(\bm{u}^{(i)}_k, C \big) = \frac{C \bm{u}^{(i)}_k}{\|\bm{u}^{(i)}_k\|}$}. // \texttt{(Normalization instead of Clipping.)}
\vspace{0.1 cm}
\STATE Send $(\bm{g}^{(i)}_k + \bm{\zeta}^{(i)}_k)$ to the server, where $\bm{\zeta}^{(i)}_k \sim \mathcal{N}(\vec{0}_d, r \sigma^2 \bm{\textup{I}}_d)$.
\vspace{0.1 cm}
\ENDFOR
\vspace{0.1 cm}
\STATE Update $\bm{w}_{k+1} \xleftarrow{} \bm{w}_k - \beta_k \bm{a}_k$, where $\bm{a}_k = \frac{1}{r}\sum_{i \in \mathcal{S}_k}(\bm{g}^{(i)}_k + \bm{\zeta}^{(i)}_k)$.
\vspace{0.1 cm}
\STATE Return $\bm{w}_\text{priv} = \bm{w}_{\tilde{k}}$, where $\tilde{k} \sim \text{unif}[0,K-1]$.
\vspace{0.1 cm}
\ENDFOR
\end{algorithmic}
\end{algorithm}

\noindent The abridged convergence result of \texttt{DP-NormFedAvg} is presented next. The full version and proof can be found in \Cref{thm2-details}.
\begin{theorem}[\textbf{Convergence of \texttt{DP-NormFedAvg}: Convex Case}]
\label{thm-new-clip-cvx-short}
In the same setting and with the same choices as \Cref{thm-clip-cvx-short}, \texttt{DP-NormFedAvg} (i.e., \Cref{alg:dp-fedavg-2}) has the following convergence guarantee:
\small
\begin{multline}
    \label{eq:thm2-1}
        \mathbb{E}\Bigg[\frac{1}{n}\sum_{i=1}^n \Bigg\{\mathbbm{1}(\|\bm{u}_{\tilde{k}}^{(i)}\| \leq \hat{C}E) \Big(2 - \frac{\rho^2 E^2}{\alpha^2}\Big)\Bigg(\frac{\hat{C} E}{\|\bm{u}_{\tilde{k}}^{(i)}\|}\Bigg)(f_i(\bm{w}_{\tilde{k}}) - f_i(\bm{w}^{*})) + \mathbbm{1}(\|\bm{u}_{\tilde{k}}^{(i)}\| > \hat{C}E) \Big(\frac{3 \hat{C} \|\bm{u}_{\tilde{k}}^{(i)}\|}{8 L E}\Big) \Bigg\}\Bigg] 
        \\
        \leq
        \hat{C} \Big(\frac{L \|\bm{w}_{0} - \bm{w}^{*}\|^2}{\gamma} + \frac{\gamma}{L}\Big) \rho + 
        {\mathbb{E}\Bigg[\frac{1}{n}\sum_{i=1}^n \mathbbm{1}(\|\bm{u}_{\tilde{k}}^{(i)}\| \leq \hat{C} E) \Bigg\{\frac{\hat{C}^2}{2 \alpha L} + \Bigg(\frac{\hat{C} E}{\|\bm{u}_{\tilde{k}}^{(i)}\|}\Bigg) \frac{\Delta_i^{*} \rho E}{\alpha^2}\Bigg\}E\Bigg]}\rho,
\end{multline}
\normalsize
with $\tilde{k} \sim \textup{unif}[0,K-1]$.
\end{theorem}

We now provide some insights on the convergence rate of \Cref{thm-new-clip-cvx-short} by comparing it with that of 
\texttt{DP-FedAvg} with clipping (i.e., \Cref{thm-clip-cvx-short}).

\subsection{Theoretical Comparison of \texttt{DP-FedAvg} with Clipping and \texttt{DP-NormFedAvg}}
\label{sec:comp-1}
Per \Cref{thm-clip-cvx-short}, recall that the convergence rate of \texttt{DP-FedAvg} with clipping (i.e., \Cref{alg:dp-fedavg}) is:
\begin{multline}
        \label{eq:feb9-1}
        \mathbb{E}\Bigg[\frac{1}{n}\sum_{i=1}^n \Bigg(\mathbbm{1}(\|\bm{u}_{\tilde{k}}^{(i)}\| \leq \hat{C}E) {\color{cyan}\Big(2 - \frac{\rho E}{\alpha} - \frac{\rho^2 E^2}{\alpha^2}\Big)} (f_i(\bm{w}_{\tilde{k}}) - f_i(\bm{w}^{*})) +  \mathbbm{1}(\|\bm{u}_{\tilde{k}}^{(i)}\| > \hat{C}E) \Big(\frac{3 \hat{C}}{8 L E} \|\bm{u}_{\tilde{k}}^{(i)}\|\Big) \Bigg)\Bigg]
        \\
        \leq \underbrace{\hat{C}\Big(\frac{L \|\bm{w}_{0} - \bm{w}^{*}\|^2}{\gamma} + \frac{\gamma}{L}\Big)\rho}_{:=\textup{A} \text{(effect of \textbf{initialization})}} + \underbrace{{\color{cyan} \mathbb{E}\Bigg[\frac{1}{n}\sum_{i=1}^n {\mathbbm{1}(\|\bm{u}_{\tilde{k}}^{(i)}\| \leq \hat{C}E)} \Big(\frac{3 E}{2 \alpha}\Big)\Bigg]} \rho}_{:=\textup{B} \text{(effect of \textbf{heterogeneity})}},
\end{multline}
with $\tilde{k} \sim \textup{unif}[0,K-1]$. In comparison, the convergence rate of \texttt{DP-NormFedAvg} (i.e., \Cref{alg:dp-fedavg-2}), under the same setting, is:
\begin{multline}
        \label{eq:feb9-2}
        \mathbb{E}\Bigg[\frac{1}{n}\sum_{i=1}^n \Bigg\{\mathbbm{1}(\|\bm{u}_{\tilde{k}}^{(i)}\| \leq \hat{C}E) \underbrace{\color{red}\Big(2 - \frac{\rho^2 E^2}{\alpha^2}\Big)}_{> (2 - \frac{\rho E}{\alpha} - \frac{\rho^2 E^2}{\alpha^2})}\underbrace{\color{red}\Bigg(\frac{\hat{C} E}{\|\bm{u}_{\tilde{k}}^{(i)}\|}\Bigg)}_{\geq 1}(f_i(\bm{w}_{\tilde{k}}) - f_i(\bm{w}^{*})) + \mathbbm{1}(\|\bm{u}_{\tilde{k}}^{(i)}\| > \hat{C}E) \Big(\frac{3 \hat{C}}{8 L E} \|\bm{u}_{\tilde{k}}^{(i)}\|\Big)\Bigg\}\Bigg]     
        \\
        \leq
        \underbrace{\hat{C} \Big(\frac{L \|\bm{w}_{0} - \bm{w}^{*}\|^2}{\gamma} + \frac{\gamma}{L}\Big) \rho}_{=\textup{A} \text{(effect of \textbf{initialization})}} + 
        \underbrace{{\color{red}\mathbb{E}\Bigg[\frac{1}{n}\sum_{i=1}^n \mathbbm{1}(\|\bm{u}_{\tilde{k}}^{(i)}\| \leq \hat{C} E) \Bigg\{\frac{\hat{C}^2}{2 \alpha L} + \Bigg(\frac{\hat{C} E}{\|\bm{u}_{\tilde{k}}^{(i)}\|}\Bigg) \frac{\Delta_i^{*} \rho E}{\alpha^2}\Bigg\}E\Bigg]} \rho}_{:=\textup{B}_2 \text{(effect of \textbf{heterogeneity})}}.
\end{multline}
The terms that are different in \cref{eq:feb9-1} and \cref{eq:feb9-2} have been colored. Let us consider the same choice of $\hat{C}$ and optimum $\bm{w}^{*}$ for both algorithms. As discussed earlier, the convergence rate depends on: (i) distance of the \textbf{initialization} $\bm{w}_0$ from the optimum $\bm{w}^{*}$ (specifically, term $\textup{A}$ in both equations), and (ii) the degree of \textbf{heterogeneity} which is itself a function of the $\Delta_i^{*}$'s (specifically, term $\textup{B}$ and $\textup{B}_2$ in \cref{eq:feb9-1} and \cref{eq:feb9-2}, respectively). 

Note that the LHS of \cref{eq:feb9-2} is larger than the LHS of \cref{eq:feb9-1}. Thus, the effect of term $\textup{A}$, i.e. the effect of initialization, on convergence is smaller in the case of normalization than clipping. Next, recalling that we must set $\hat{C} \geq 4\sqrt{L \max_{j \in [n]} \Delta_j^{*}}$ in both Theorem \ref{thm-clip-cvx-short} and \ref{thm-new-clip-cvx-short}, let us choose $\hat{C} = c \sqrt{L \max_{j \in [n]} \Delta_j^{*}}$ with $c \geq 4$ in both cases. Then:
\begin{equation}
    \textup{B}_2 = {\mathbb{E}\Bigg[\frac{1}{n}\sum_{i=1}^n \mathbbm{1}(\|\bm{u}_{\tilde{k}}^{(i)}\| \leq \hat{C} E) \Bigg\{\Bigg(\frac{c^2 E}{2 \alpha}\Bigg) \max_{j \in [n]} \Delta_j^{*} + \Bigg(\frac{\hat{C} E}{\|\bm{u}_{\tilde{k}}^{(i)}\|}\Bigg) \Bigg(\frac{\Delta_i^{*} \rho E^2}{\alpha^2}\Bigg)\Bigg\}\Bigg]} \rho.
\end{equation}
Now observe that for $c \geq 4$, $\textup{B} < \textup{B}_2$. However, the LHS of \cref{eq:feb9-2} is more than that of \cref{eq:feb9-1}. So in general, it is difficult to predict whether the effect of heterogeneity on convergence is smaller in the case of clipping or normalization. 

{But the effect of heterogeneity can be mitigated arbitrarily for both clipping and normalization by increasing $\alpha$, i.e. increasing the number of rounds $K$ arbitrarily (recall that we set $K = \big(\frac{2 \alpha \gamma}{\hat{C} E}\big) \frac{1}{\rho^2}$). 
So asymptotically, i.e. for $\alpha \to \infty$ or $K \to \infty$, the effect of heterogeneity gets killed and only the effect of initialization matters, where we expect normalization to outperform clipping.} It is worth mentioning here that the previous discussion is not specific to the federated setting and also applies to the centralized setting (i.e., $E=1$).
\\
\\
We summarize all the above discussion in the following remark.
\begin{remark}[\textbf{Normalization versus Clipping}]
\label{rmk-cmp}
Compared to clipping, normalization is associated with a smaller effect of initialization on convergence. However, in general, it is difficult to characterize whether the effect of heterogeneity is smaller for normalization or clipping. {The good thing is that for both clipping and normalization, the effect of heterogeneity gets killed asymptotically, i.e. when the number of communication rounds ($K$) tends to $\infty$.}

Hence, for problems that do not have a high degree of heterogeneity and the effect of initialization is more severe (for e.g., by poor random initialization) {and/or if we can train for a very large number of rounds}, we expect normalization to offer better convergence than clipping in private optimization. 
\end{remark}

It is also worth pointing out that clipping can be equivalent to normalization in certain scenarios. Specifically, suppose the client update norms are lower bounded by ${C}_{\text{low}}$; then, clipping with threshold $C \leq {C}_{\text{low}}$ is equivalent to normalization with the same scaling factor.
\\
\\
In \Cref{sec:dp-normfed}, we provide a more intuitive argument as to why update normalization can offer better convergence than update clipping in terms of their signal (viz., update
norm) to noise ratios, and also relate it to the previous theoretical comparison in \Cref{sec:comp-1}. 

\subsection{Intuitive Explanation of why Normalization can Outperform Clipping}
\label{sec:dp-normfed}
Intuitively, clipping has the following issue with respect to optimization - as the client update norms decrease and fall below the clipping threshold, the norm of the added noise (which has constant expectation proportional to the clipping threshold, regardless of the client update norms) can become arbitrarily larger than the client update norms, which should inhibit convergence. This issue is not as grave in \texttt{DP-NormFedAvg} because its update-\textit{normalization} step ensures that the noise norm cannot become arbitrarily larger than the \textit{normalized} update's norm (even if the original update's norm is small). In other words, the signal (which is the update
norm) to noise ratio of clipping eventually falls below that of normalization.

{The mathematical manifestation of the aforementioned argument can be also seen in the convergence bounds of clipping (i.e., \cref{eq:feb9-1}) and normalization (i.e., \cref{eq:feb9-2}) in \Cref{sec:comp-1}. Specifically, note that the coefficient of $\mathbbm{1}(\|\bm{u}_{\tilde{k}}^{(i)}\| \leq \hat{C}E)$ (i.e., when the update norm $\|\bm{u}_{\tilde{k}}^{(i)}\|$ is less than or equal to the clipping threshold $\hat{C}E$) in the LHS of \cref{eq:feb9-2} is at least $\frac{\hat{C}E}{\|\bm{u}_{\tilde{k}}^{(i)}\|}$ ($\geq 1$) times more than the corresponding term in the LHS of \cref{eq:feb9-1}; this amplification is a 
consequence of the improvement in signal to noise ratio (SNR) of normalization over clipping. 
On the other hand, the coefficient of $\mathbbm{1}(\|\bm{u}_{\tilde{k}}^{(i)}\| > \hat{C}E)$ (i.e., when the update norm is more than the clipping threshold) in the LHS of \cref{eq:feb9-2} is exactly the same as the corresponding term in the LHS of \cref{eq:feb9-1}; this is because normalization and clipping are equivalent when $\|\bm{u}_{\tilde{k}}^{(i)}\| > \hat{C}E$. Now, as discussed in \Cref{sec:comp-1}, the RHS of both \cref{eq:feb9-1} and \cref{eq:feb9-2} become the same asymptotically with a large number of rounds as the effect of heterogeneity dies down. Thus, the asymptotic convergence of normalization (i.e., \cref{eq:feb9-2}) is better than that of clipping (i.e., \cref{eq:feb9-1}).}
\rudrajit{Try to improve and shorten the above rationale?}

Let us now see some experimental results on a synthetic problem illustrating the superiority of normalization over clipping. 

\subsection{Empirical Comparison of \texttt{DP-FedAvg} with Clipping and \texttt{DP-NormFedAvg} on a Synthetic Problem}
\label{sec:norm-vs-clip-expt}
We consider $f_i(\bm{w}) = \frac{1}{2}(\bm{w} - \bm{w}_i^{*})^T \bm{Q}_i (\bm{w} - \bm{w}_i^{*})$, where $i \in [100]$ (so, $n = 100$) and $\bm{w} \in \mathbb{R}^{200}$ (so, $d = 200$). Further, $\bm{w}_i^{*}$ is drawn i.i.d. from $\mathcal{N}(0,\textup{I}_{200})$ and $\bm{Q}_i = \bm{A}_i\bm{A}_i^T$, where $\bm{A}_i$ is a $200 \times 20$ matrix whose entries are drawn i.i.d from $\mathcal{N}(0,\frac{1}{20^2})$; hence, $\bm{Q}_i$ is a PSD matrix with bounded maximum eigenvalue, due to which $f_i$ is convex and smooth. 

We set $(\varepsilon, \delta) = (5, 10^{-6})$, $K=500$ and $E=20$ for this set of experiments. We consider two different initializations with different distances from the global optimum $\bm{w}^{*}$: 
\begin{itemize}
    \item \textbf{I1}: $\bm{w}_0 = \bm{w}^{*} + \bm{z}$, and 
    \item \textbf{I2}: $\bm{w}_0 = \bm{w}^{*} + \frac{\bm{z}}{5}$, 
\end{itemize}
where each coordinate of $\bm{z}$ is drawn i.i.d. from the continuous uniform distribution with support (0,1). We set $\eta_k = \beta_k = \eta$ for all rounds $k$, and also have full-device participation. In \Cref{fig:1}, we plot the function suboptimality (i.e., $f(\bm{w}_k) - f(\bm{w}^{*})$ at round number $k$) of \texttt{DP-FedAvg} with Clipping and \texttt{DP-NormFedAvg} for different values of $\eta$ and clipping threshold/scaling factor $C$, for {I1} and {I2}; specifically, \enquote{Clip($\eta$)} and \enquote{Norm($\eta$)} in the legend denote \texttt{DP-FedAvg} with Clipping and \texttt{DP-NormFedAvg} with $\eta_k = \beta_k = \eta$ for all rounds $k$, respectively. 
In \Cref{fig:2}, for each round $k$, we plot the corresponding $\text{SNR} := \frac{\big\|\frac{1}{r}\sum_{i \in \mathcal{S}_k}\bm{g}^{(i)}_k\big\|}{\big\|\frac{1}{r}\sum_{i \in \mathcal{S}_k}\bm{\zeta}^{(i)}_k\big\|}$, where $\bm{g}^{(i)}_k$ and $\bm{\zeta}^{(i)}_k$ are the clipped/normalized per-client update and per-client noise, respectively, as defined in \Cref{alg:dp-fedavg}/\ref{alg:dp-fedavg-2}. We only show the SNR plots for one value of $\eta$ as the trend for other values of $\eta$ is similar (and to avoid congestion).
All plots are averaged over three independent runs. For a fair comparison, in each run, the exact same noise vectors (sampled randomly at each round) are used in both algorithms.

The thing to note in \Cref{fig:1} is that for $C = \{50, 100\}$ and all values of $\eta$, normalization attains an appreciably lower function suboptimality than clipping. For $C = 40$, normalization is just slightly better. The SNR values in \Cref{fig:2} also follow a similar trend -- the improvement in SNR for normalization compared to clipping is much higher for $C = \{50, 100\}$ than $C=40$.
We only show results up to $C = 40$ as for smaller values of $C$, clipping and normalization are equivalent. As discussed at the end of \Cref{sec:comp-1} after \Cref{rmk-cmp}, recall that if the client update norms are lower bounded by ${C}_{\text{low}}$, then clipping with threshold $C \leq {C}_{\text{low}}$ is equivalent to normalization with the same scaling factor.

For further illustration, in \Cref{fig:traj}, we plot the smoothed 2D projection of the trajectories of \texttt{DP-FedAvg} with clipping and \texttt{DP-NormFedAvg} for two of the cases of \Cref{fig:1}. 
From here, we can see that \texttt{DP-NormFedAvg} reaches closer to the optimum than \texttt{DP-FedAvg} with clipping.
\rudrajit{Add details about how these trajectories are generated? Also remember to change the caption of \Cref{fig:traj} should anything change.}

These plots corroborate our previous theoretical predictions and intuition. {We also show the superiority of normalization over clipping via experiments on actual datasets in \Cref{sec:expts}.}

\begin{figure*}[!htb]
\centering 
\subfloat[\textbf{I1}: ${C}=40$]{
    \label{fig:1_a}
	\includegraphics[width=0.33\textwidth]{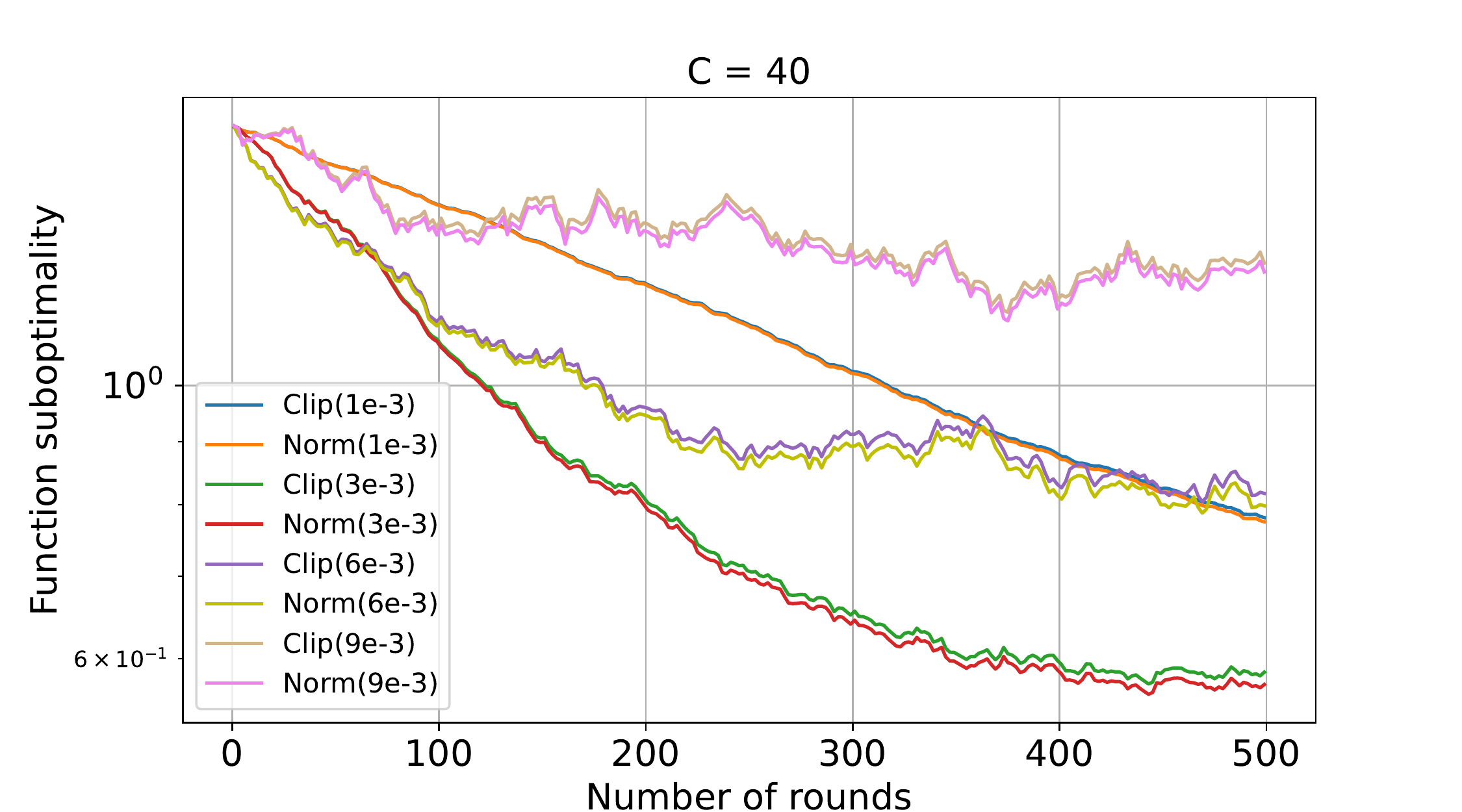}
	} 
\subfloat[\textbf{I1}: ${C}=50$]{
    \label{fig:1_b}
	\includegraphics[width=0.33\textwidth]{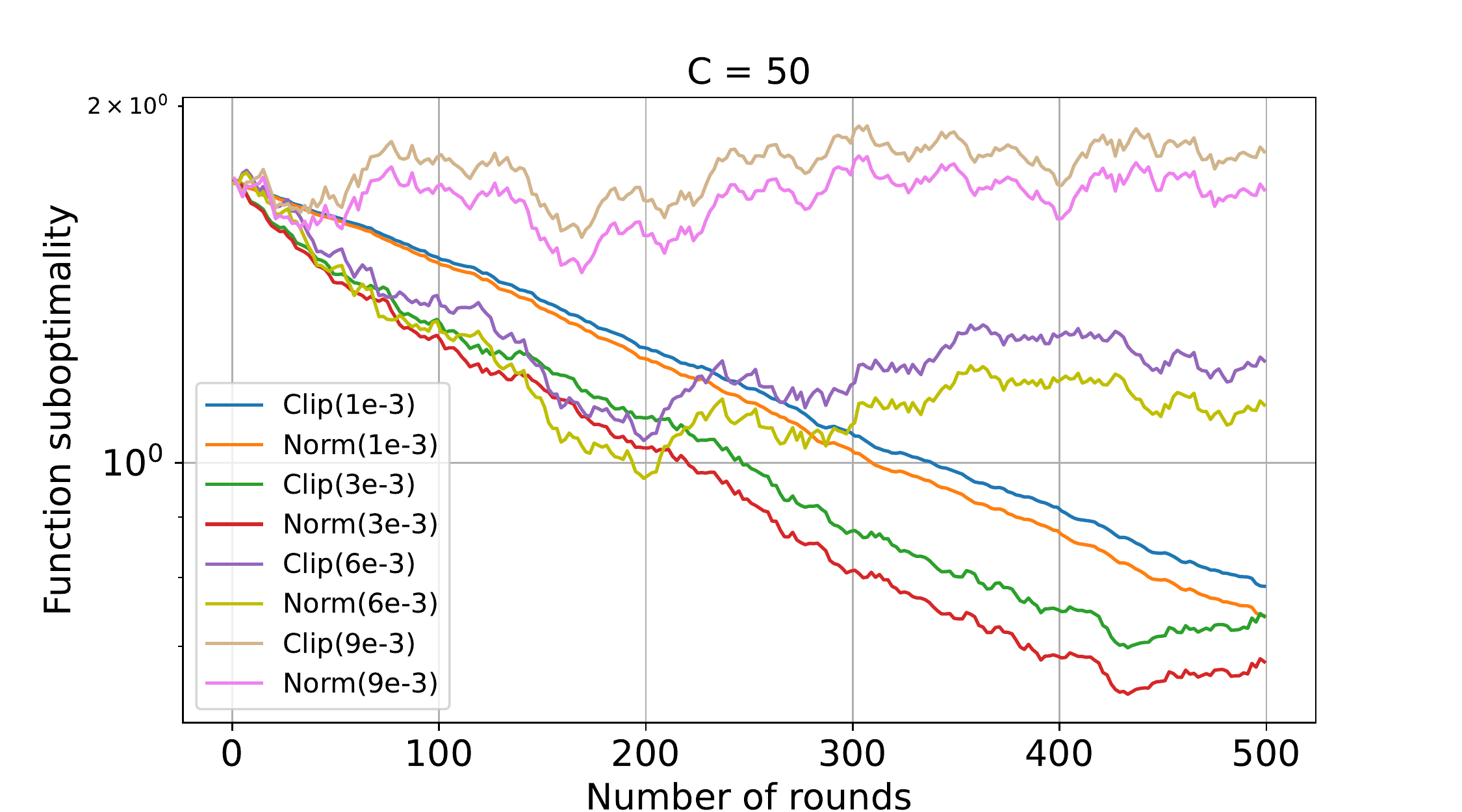}
	} 
\subfloat[\textbf{I1}: ${C}=100$]{
    \label{fig:1_c}
	\includegraphics[width=0.33\textwidth]{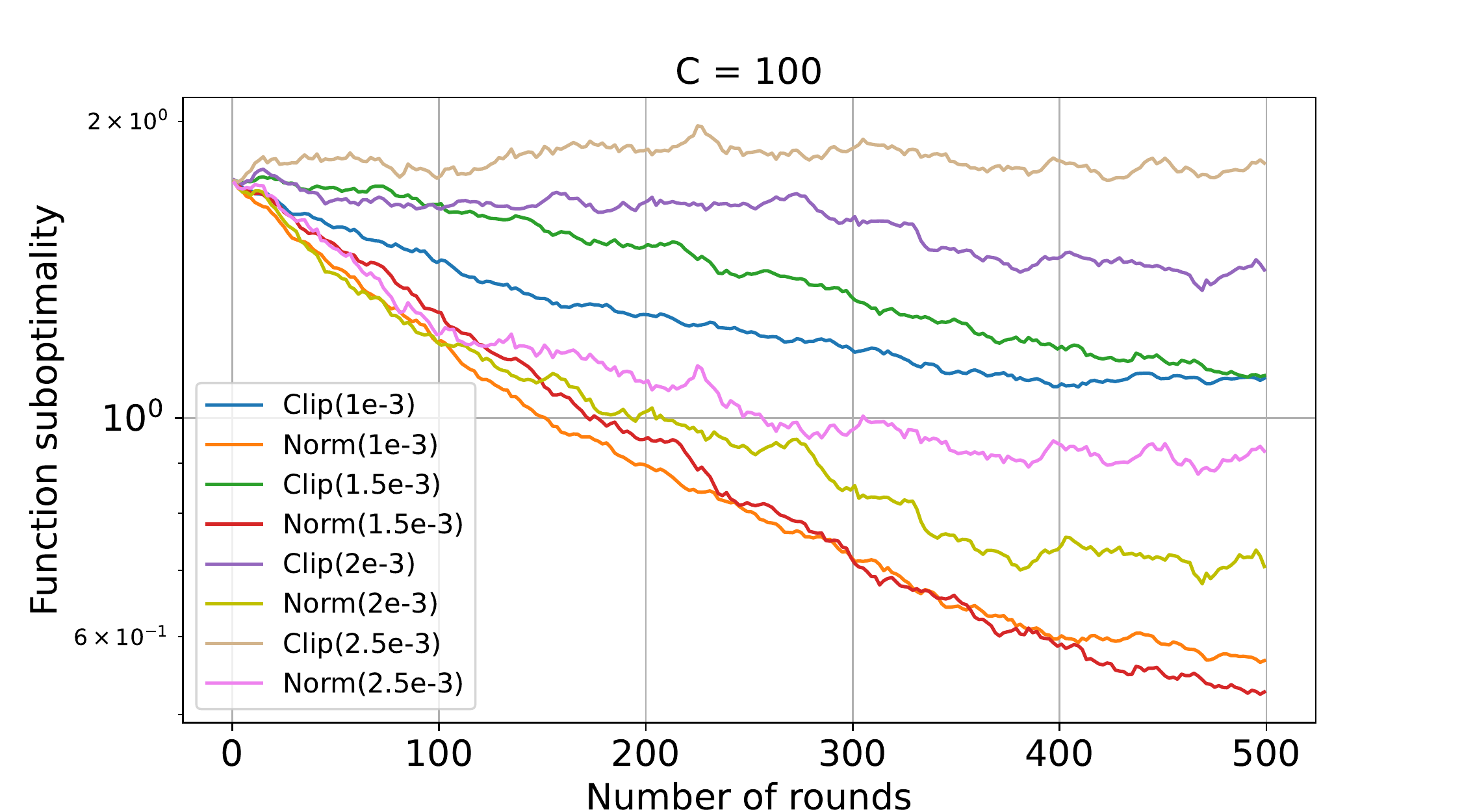}
	} 
\\
\subfloat[\textbf{I2}: ${C}=40$]{
    \label{fig:1_d}
	\includegraphics[width=0.33\textwidth]{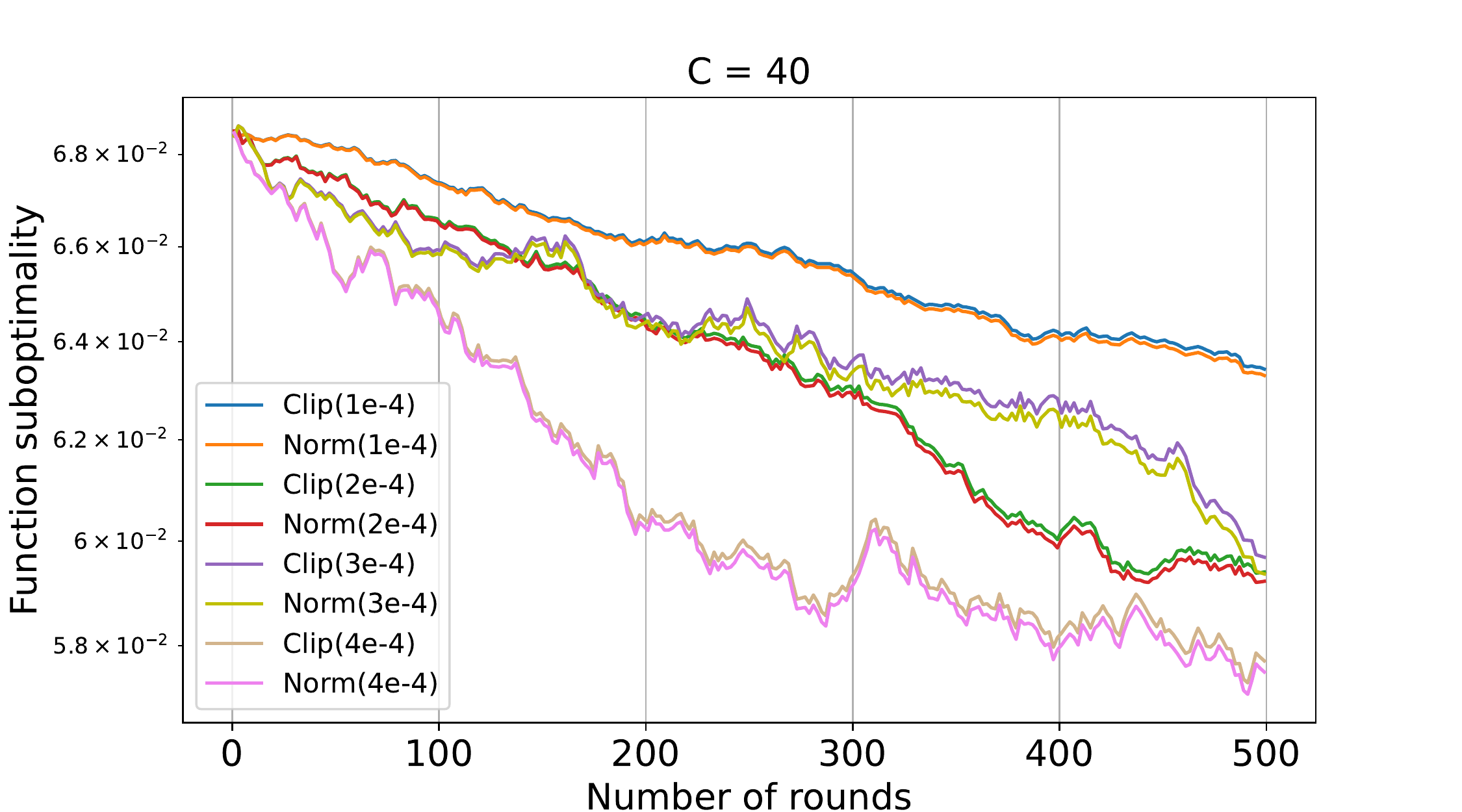}
	} 
\subfloat[\textbf{I2}: ${C}=50$]{
    \label{fig:1_e}
	\includegraphics[width=0.33\textwidth]{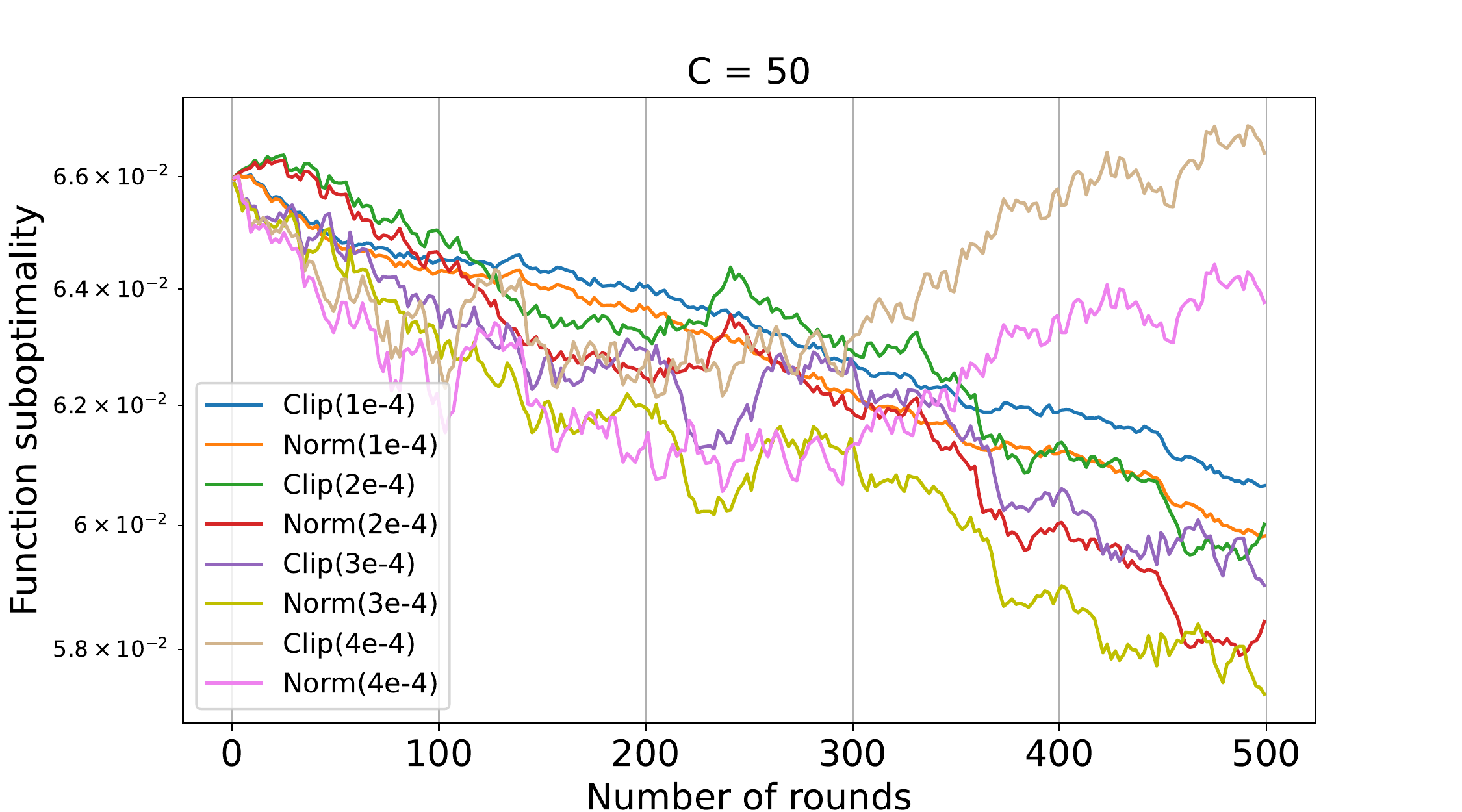}
	} 
\subfloat[\textbf{I2}: ${C}=100$]{
    \label{fig:1_f}
	\includegraphics[width=0.33\textwidth]{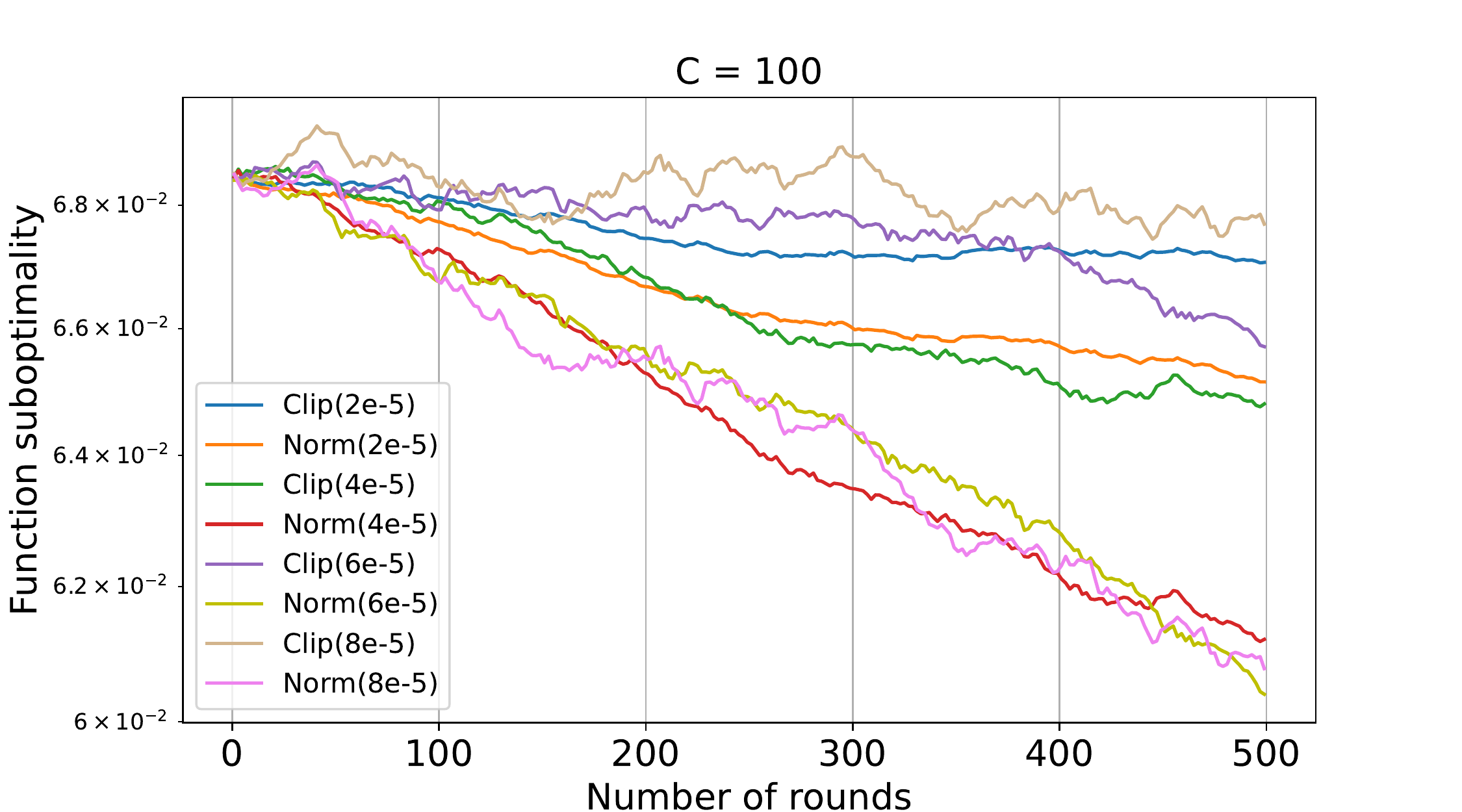}
	} 
\caption{Function suboptimality (i.e., $f(\bm{w}_k) - f(\bm{w}^{*})$ at round number $k$) of \texttt{DP-FedAvg} with Clipping and \texttt{DP-NormFedAvg} for different values of $\eta$ (recall, $\eta_k = \beta_k = \eta$ for all rounds $k$) and clipping threshold/scaling factor $C$, for {I1} and {I2} described in \Cref{sec:norm-vs-clip-expt}. Specifically, \enquote{Clip($\eta$)} and \enquote{Norm($\eta$)} denote \texttt{DP-FedAvg} with Clipping and \texttt{DP-NormFedAvg} with $\eta_k = \beta_k = \eta$, respectively. All plots are averaged over three independent runs.
\\
For $C = \{50,100\}$ and all values of $\eta$, normalization does significantly better than clipping. For $C = 40$ (and lower), normalization and clipping are nearly equivalent, but clipping never does better than normalization. This validates our theoretical predictions in \Cref{sec:comp-1}.}
\label{fig:1}
\end{figure*}

\begin{figure*}[!htb]
\centering 
\subfloat[\textbf{I1}: ${C}=40$]{
    \label{fig:2_a}
	\includegraphics[width=0.33\textwidth]{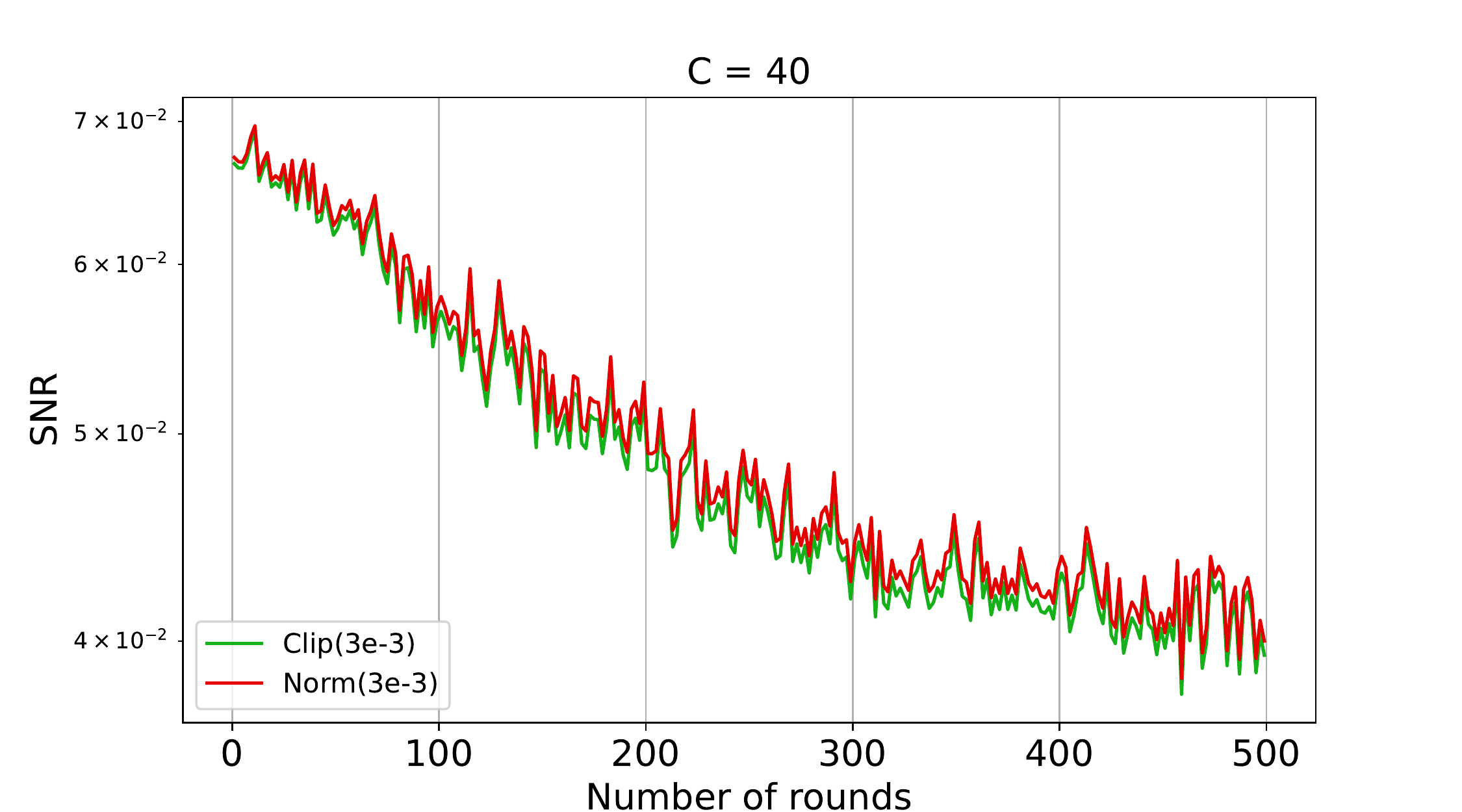}
	} 
\subfloat[\textbf{I1}: ${C}=50$]{
    \label{fig:2_b}
	\includegraphics[width=0.33\textwidth]{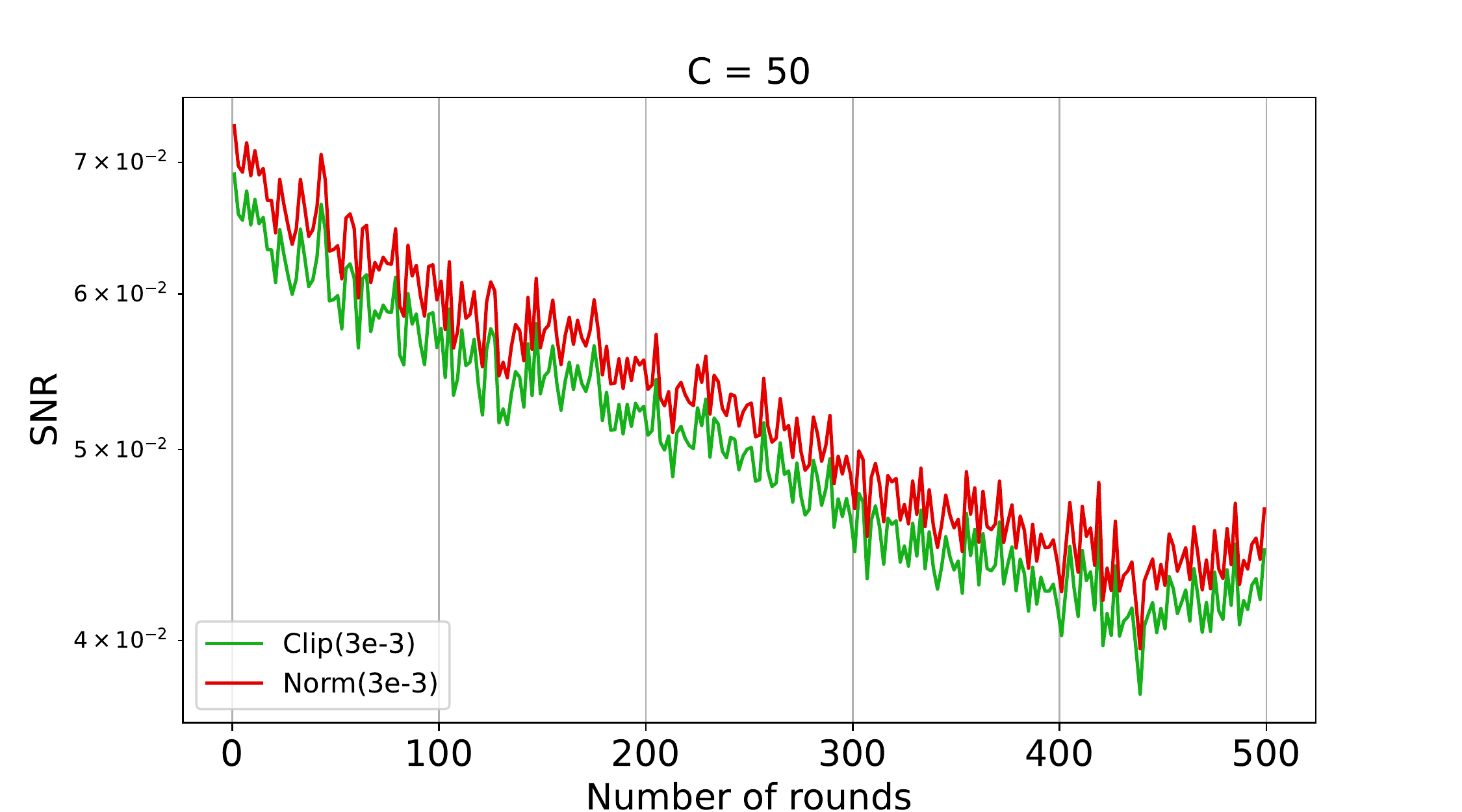}
	} 
\subfloat[\textbf{I1}: ${C}=100$]{
    \label{fig:2_c}
	\includegraphics[width=0.33\textwidth]{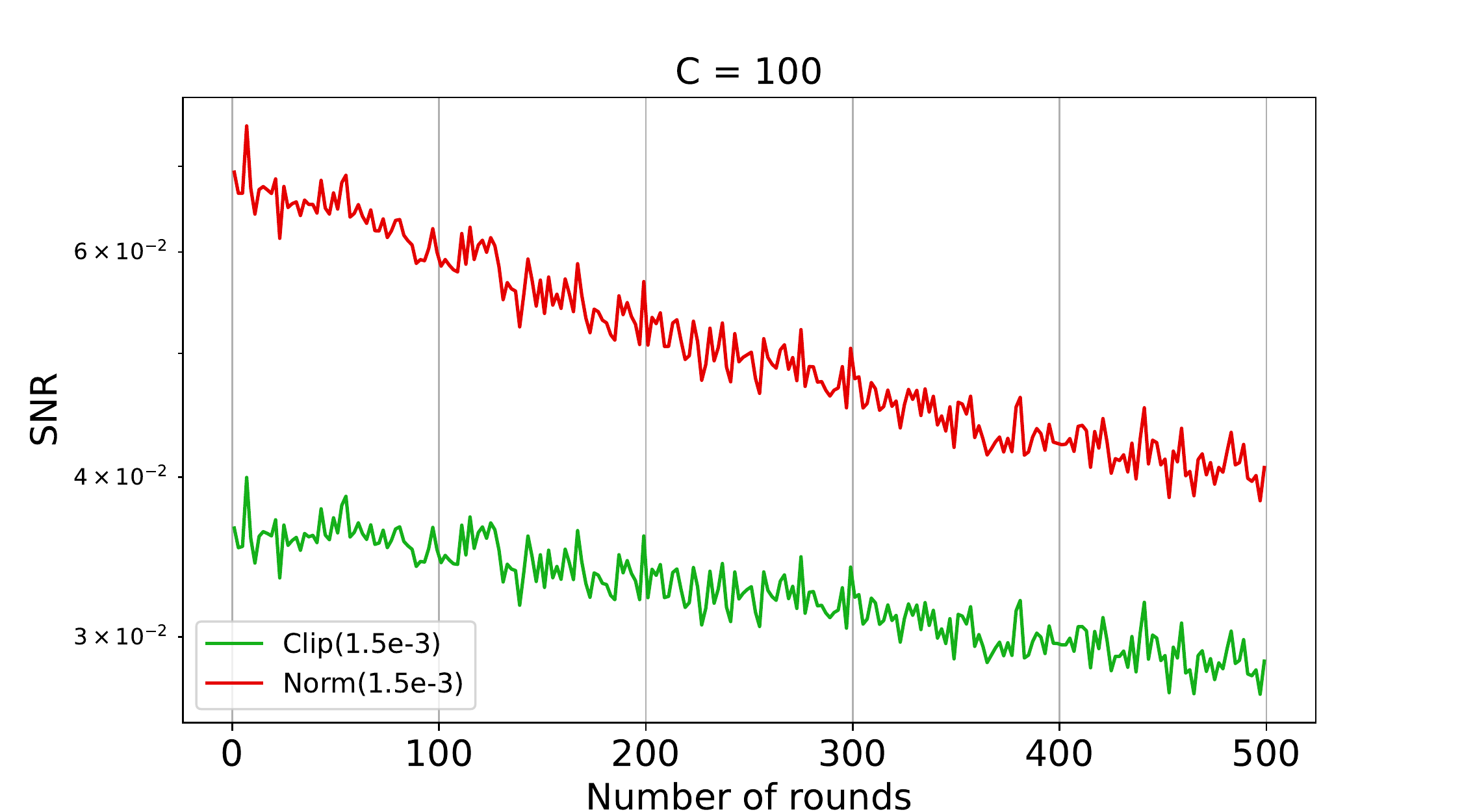}
	} 
\\
\subfloat[\textbf{I2}: ${C}=40$]{
    \label{fig:2_d}
	\includegraphics[width=0.33\textwidth]{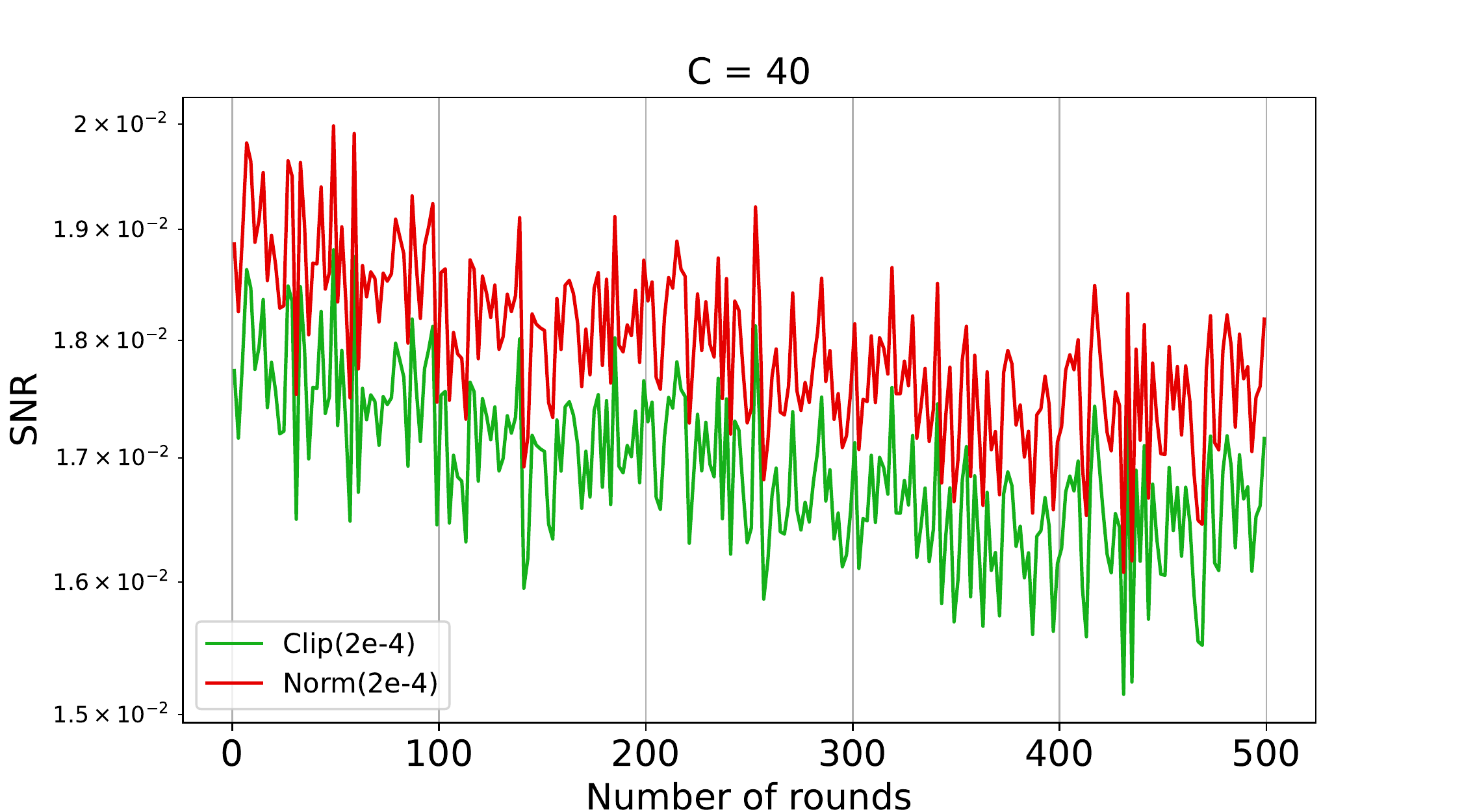}
	} 
\subfloat[\textbf{I2}: ${C}=50$]{
    \label{fig:2_e}
	\includegraphics[width=0.33\textwidth]{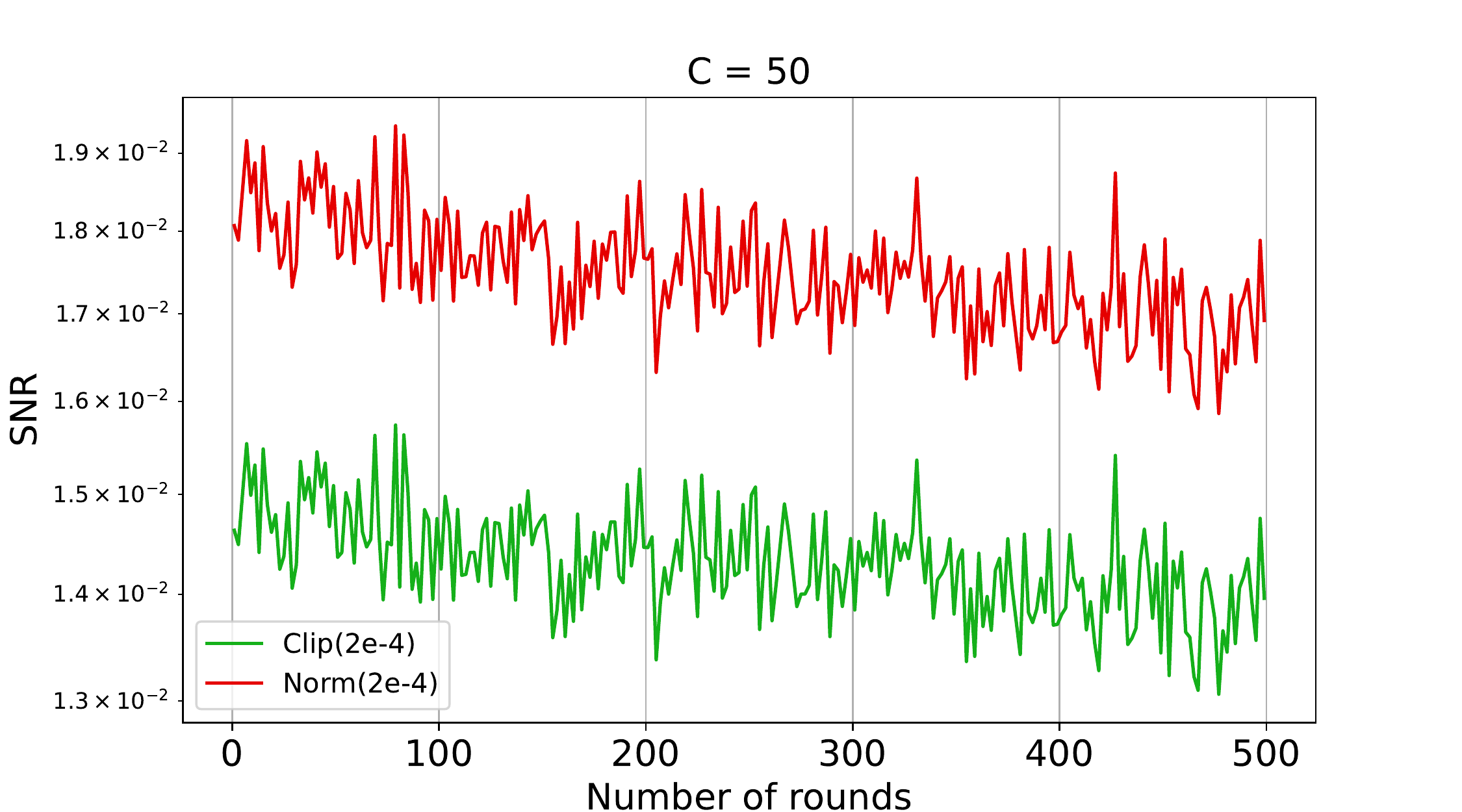}
	} 
\subfloat[\textbf{I2}: ${C}=100$]{
    \label{fig:2_f}
	\includegraphics[width=0.33\textwidth]{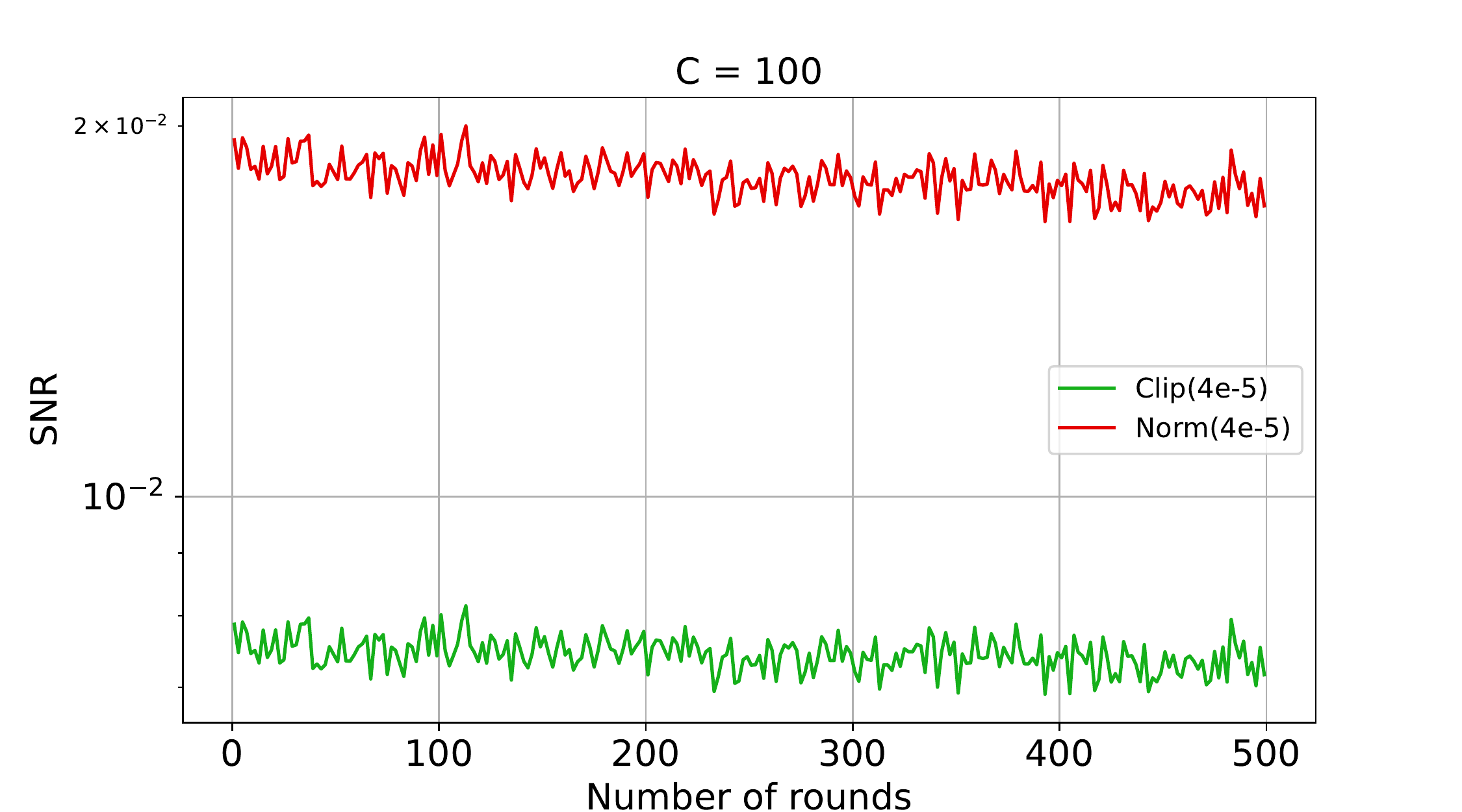}
	} 
\caption{In the same setting and with the same notation as \Cref{fig:1}, comparison of $\text{SNR} := \frac{\big\|\frac{1}{r}\sum_{i \in \mathcal{S}_k}\bm{g}^{(i)}_k\big\|}{\big\|\frac{1}{r}\sum_{i \in \mathcal{S}_k}\bm{\zeta}^{(i)}_k\big\|}$, where $\bm{g}^{(i)}_k$ and $\bm{\zeta}^{(i)}_k$ are the clipped/normalized per-client update and per-client noise, as defined in \Cref{alg:dp-fedavg}/\ref{alg:dp-fedavg-2}. The SNR for only one value of $\eta$ is shown here as the trend for other values of $\eta$ is similar. 
\\
As per our discussion in \Cref{sec:dp-normfed}, the SNR of normalization is never lower than that of clipping, explaining the superiority of the former.
}
\label{fig:2}
\end{figure*}

\begin{figure*}[!htb]
\centering 
\subfloat[\textbf{I1}: ${C}=50$ and $\eta = 0.003$]{
    \label{fig:traj_a}
	\includegraphics[width=0.45\textwidth]{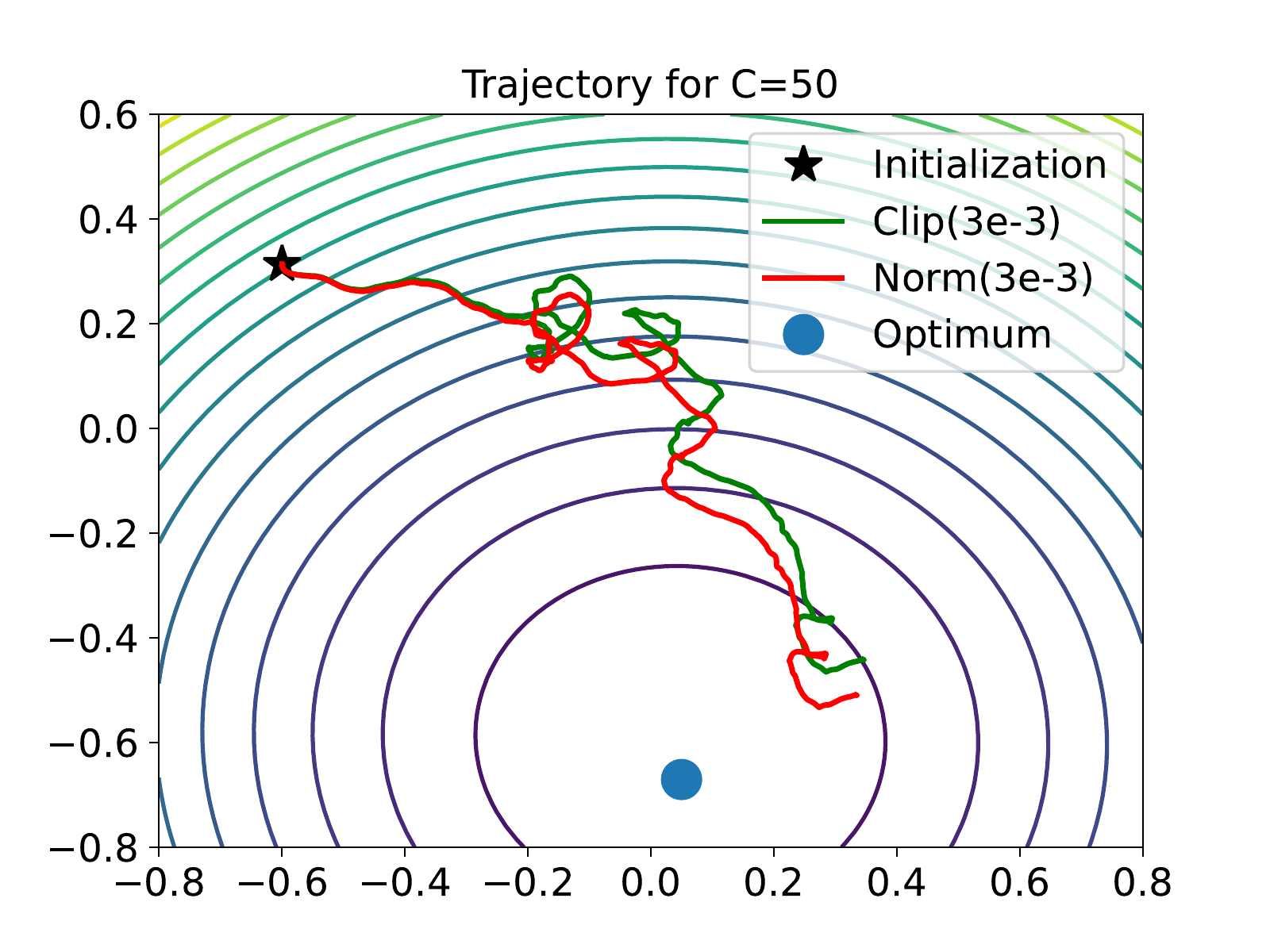}
	} 
\subfloat[\textbf{I1}: ${C}=100$ and $\eta = 0.001$]{
    \label{fig:traj_b}
	\includegraphics[width=0.45\textwidth]{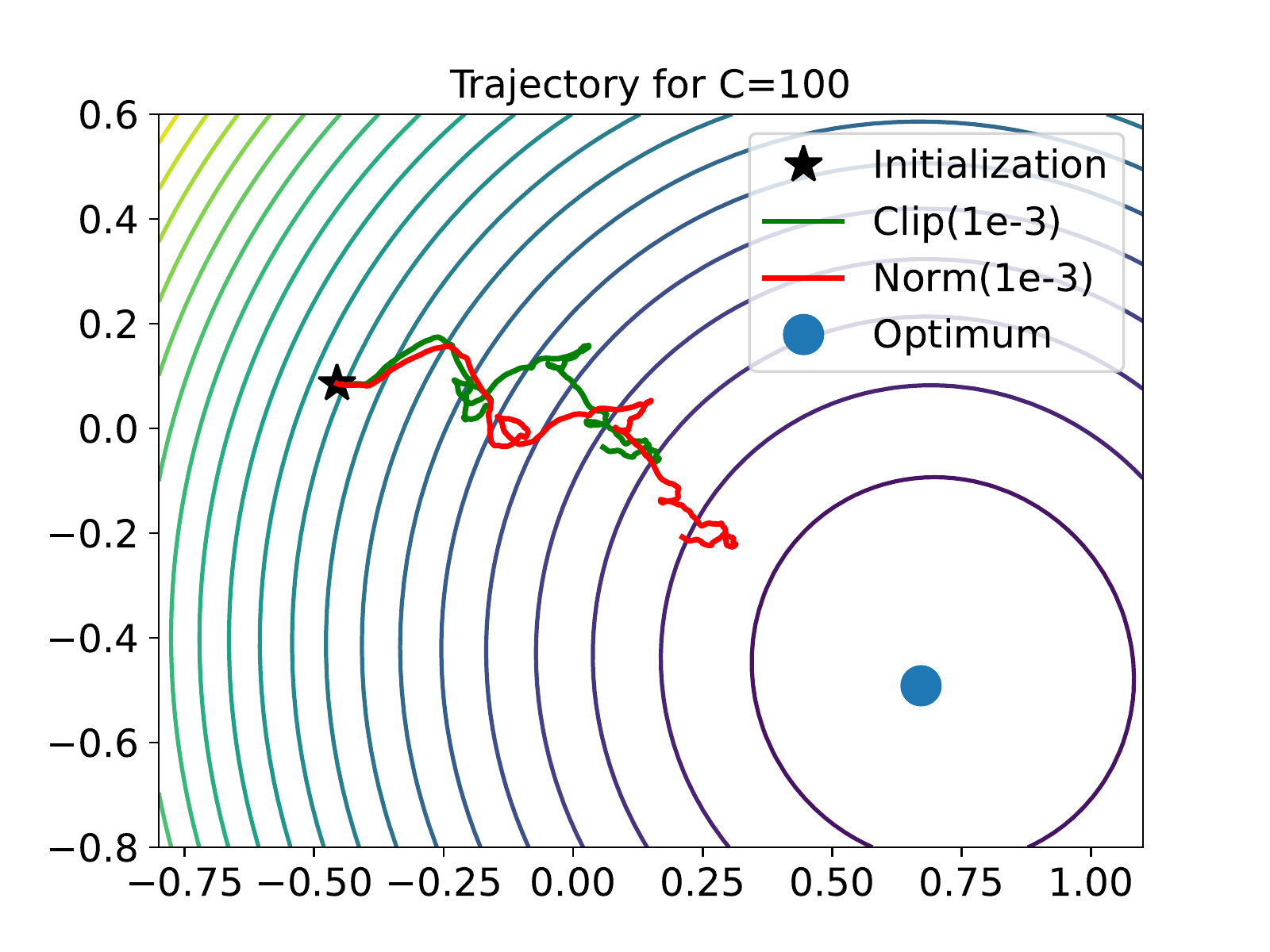}
	} 
\caption{Smoothed 2D projections of the trajectories of \texttt{DP-FedAvg} with clipping and \texttt{DP-NormFedAvg} for two of the cases considered in \Cref{fig:1}. Observe that \texttt{DP-NormFedAvg} reaches closer to the optimum than \texttt{DP-FedAvg} with clipping.}
\label{fig:traj}
\end{figure*}

\section{Experiments}
\label{sec:expts}
{We consider the task of \textit{private} multi-class classification to compare \texttt{DP-FedAvg} with clipping against \texttt{DP-NormFedAvg}; for brevity, we will often call them just clipping and normalization, respectively. Our experiments are performed on three benchmarking datasets, Fashion-MNIST \cite{xiao2017fashion} (abbreviated as FMNIST henceforth), CIFAR-10 and CIFAR-100, where the first two datasets have $10$ classes each and the last one has $100$ classes.} 

{Specifically, we consider logistic regression on FMNIST, CIFAR-10 and CIFAR-100 with $\ell_2$-regularization; the weight decay value in PyTorch for $\ell_2$-regularization is set to 1e-4. For FMNIST, we flatten each image into a $784$-dimensional vector and use that as the feature vector. For CIFAR-10 and CIFAR-100, we use 512-dimensional features extracted from the last layer of a ResNet-18 \cite{he2016deep} model pretrained on ImageNet.
Similar to \cite{mcmahan2017communication}, we simulate a heterogeneous setting by distributing the data among the clients such that each client can have data from at most five classes. {The exact procedure is described in \Cref{sec:expt_det}.}
For the CIFAR-10 and CIFAR-100 (respectively, FMNIST) experiment, the number of clients $n$ is set to 5000 (respectively, 3000), with each client having the same number of samples. 
The number of participating clients in each round is set to $r = 0.2n$ for all datasets, with 20 local client updates per-round. 

We consider two privacy levels: $\varepsilon = \{5, 1.5\}$ with $\delta = 10^{-5}$; note that $\varepsilon  = 5$ (respectively, $1.5$) corresponds to the low (respectively, high) privacy regime. For clipping and normalization, the values of $C$ that we tune over are $\{500,250,125,62.5,31.25,15.625\}$. {The details about the learning rate schedule can be found in \Cref{sec:expt_det}.} In \Cref{tab1}, we show the comparison between clipping and normalization (in terms of test accuracy) for the two aforementioned privacy levels as well as vanilla \texttt{FedAvg} (without any privacy) as the baseline. The results reported here are the best ones for each algorithm by tuning over $C$ and the learning rates, and have been averaged over three different runs. 

{In all cases, normalization is clearly superior to clipping. It is worth noting that the improvement obtained with normalization is more for the low privacy regime (i.e., $\varepsilon = 5$).
\rudrajit{More comments and experiments?}
}
}

\begin{table}[!htb]
\centering
\begin{subtable}[c]{0.5\textwidth}
\centering
\begin{tabular}{|l|c|r|}
\hline
Algo. & $(5,10^{-5})$-DP & $(1.5,10^{-5})$-DP
\\
\hline
Clipping & 75.59\% & 56.90\%
\\
\hline
Normalization & \textbf{77.72}\% & \textbf{57.80}\%
\\
\hline
FedAvg (w/o privacy) & \multicolumn{2}{|c|}{\texttt{83.43}\%}
\\
\hline
\end{tabular}
\subcaption{FMNIST}
\end{subtable}
\\
\vspace{0.2 cm}
\begin{subtable}[c]{0.5\textwidth}
\centering
\begin{tabular}{|l|c|r|}
\hline
Algo. & $(5,10^{-5})$-DP & $(1.5,10^{-5})$-DP
\\
\hline
Clipping & 82.63\% & 81.53\%
\\
\hline
Normalization & \textbf{84.21\%} & \textbf{82.42\%}
\\
\hline
FedAvg (w/o privacy) & \multicolumn{2}{|c|}{\texttt{85.64}\%}
\\
\hline
\end{tabular}
\subcaption{CIFAR-10}
\end{subtable}
\\
\vspace{0.2 cm}
\begin{subtable}[c]{0.5\textwidth}
\centering
\begin{tabular}{|l|c|r|}
\hline
Algo. & $(5,10^{-5})$-DP & $(1.5,10^{-5})$-DP
\\
\hline
Clipping & 56.53\% & 41.33\%
\\
\hline
Normalization & \textbf{59.36\%} & \textbf{42.76\%}
\\
\hline
FedAvg (w/o privacy) & \multicolumn{2}{|c|}{\texttt{64.61}\%}
\\
\hline
\end{tabular}
\subcaption{CIFAR-100}
\end{subtable}
\\
\caption{Average {test accuracy} over the last 5 rounds for (a) FMNIST, (b) CIFAR-10 and (c) CIFAR-100. Recall that \enquote{Clipping} and \enquote{Normalization} denote \texttt{DP-FedAvg} with Clipping and \texttt{DP-NormFedAvg}, respectively. The accuracy of \texttt{FedAvg}, which is our baseline \textit{without privacy}, is at the bottom.}
\label{tab1}
\end{table}

\section{Conclusion}
In this work, we provide the first convergence result for \texttt{DP-FedAvg} with clipping (which is the most standard algorithm for differentially private FL) in the convex case, and without assuming Lipschitzness. We also propose \texttt{DP-NormFedAvg} which normalizes client updates rather than clipping them (which is the customary approach for bounding sensitivity). Theoretically, we argue that \texttt{DP-NormFedAvg} should have better convergence than \texttt{DP-FedAvg} with clipping for problems that do not have a high degree of heterogeneity and the effect of poor initialization is more severe, and/or if we can train for a large number of rounds. Intuitively, this happens because
normalization has a higher signal (i.e., update norm) to noise ratio than clipping. We also show the superiority of normalization over clipping via several experiments. 

Several avenues of future work are possible. One of them is to provide principled recommendations on how to set the clipping threshold. Another one is to explore the feasibility of using adaptive and/or round-dependent clipping thresholds. It would be also nice to come up with meaningful additional assumptions that hold in practice, in order to simplify and/or improve our convergence results.
\section{Acknowledgement}
This work is supported in part by NSF grants CCF-1564000, IIS-1546452 and HDR-1934932.

\bibliography{bib}
\bibliographystyle{alpha}

\clearpage

\appendix

\clearpage
\onecolumn
\appendix

\begin{center}
    \textbf{\Large Appendix}\vspace{5mm}
\end{center}

\section{The \texttt{FedAvg} Algorithm}
\label{fedavg-sec}
For the sake of completeness, here we state the famous \texttt{FedAvg} algorithm of  \cite{mcmahan2017communication} (with local updates using full gradients).
\begin{algorithm}[!htb]
	\caption{\texttt{FedAvg} \cite{mcmahan2017communication}}
	\label{alg:fedavg}
	\begin{algorithmic}[1]
		\STATE {\bfseries Input:} 
		Initial point $\bm{w}_0$, number of rounds of communication $K$, number of local updates per round $E$, local learning rates  $\{\eta_{k}\}_{k=0}^{K-1}$ and number of participating clients in each round $r$.
		\vspace{0.1 cm}
		\FOR{$k =0,\dots, K-1$}
		\vspace{0.2 cm}
		\STATE Server sends $\bm{w}_k$ to a random set $\mathcal{S}_k$ of $r$ clients chosen uniformly at random.
		\vspace{0.2 cm}
		\FOR{client $i \in \mathcal{S}_k$}
		\vspace{0.2 cm}
		\STATE Set $\bm{w}_{k,0}^{(i)} = \bm{w}_k$.
		\vspace{0.2 cm}
		\FOR{$\tau = 0,\ldots,E-1$}
		\vspace{0.2 cm}
		\STATE Update $\bm{w}_{k,\tau+1}^{(i)} \xleftarrow{} \bm{w}_{k,\tau}^{(i)} - \eta_{k} \nabla f_i(\bm{w}_{k,\tau}^{(i)})$.
		\vspace{0.2 cm}
		\ENDFOR
		\vspace{0.2 cm}
		\STATE Send $\bm{w}_{k} - \bm{w}_{k,E}^{(i)}$ to the server.
		\label{line:fedavg-1}
		\vspace{0.2 cm}
		\ENDFOR
		\vspace{0.2 cm}
		\STATE Update 
		$\bm{w}_{k+1} \xleftarrow{} \bm{w}_{k} - \frac{1}{r}\sum_{i \in \mathcal{S}_k}(\bm{w}_{k} - \bm{w}_{k,E}^{(i)})$. 
		\\
		// (\texttt{The above is equivalent to $\bm{w}_{k+1} \xleftarrow{} \frac{1}{r}\sum_{i \in \mathcal{S}_k}\bm{w}_{k,E}^{(i)}$, so the clients might as well just send the $\bm{w}_{k,E}^{(i)}$'s.})
		\label{line:fedavg-2}
		\vspace{0.2 cm}
		\ENDFOR
	\end{algorithmic}
\end{algorithm}

\section{Full Version of Theorem~\ref{thm-clip-cvx-short} and its Proof}
\label{thm1-details}
\begin{theorem}[\textbf{Full version of \Cref{thm-clip-cvx-short}}]
\label{thm-clip-cvx}
Suppose each $f_i$ is convex and $L$-smooth over $\mathbb{R}^d$. Let $\hat{C} := \frac{C}{E}$, where $C$ is the clipping threshold used in \Cref{alg:dp-fedavg}. For any $\bm{w}^{*} \in \arg \min_{\bm{w}' \in \mathbb{R}^d} f(\bm{w}')$ and $\Delta_i^{*} := f_i(\bm{w}^{*}) - \min_{\bm{w}' \in \mathbb{R}^d} f_i(\bm{w}') \geq 0$, \Cref{alg:dp-fedavg} with $\hat{C} \geq 4 \sqrt{L \max_{j \in [n]} \Delta_j^{*}}$, $\beta_k = \eta_k = \eta = \big(\frac{\gamma}{\hat{C} L E K}\big) \frac{1}{\rho}$ and $K > \big(\frac{2 \gamma}{\hat{C}E}\big) \frac{1}{\rho}$, where $\gamma > 0$ is a constant of our choice, has the following convergence guarantee:
\small
\begin{multline*}
    \mathbb{E}\Bigg[\frac{1}{n}\sum_{i=1}^n \Bigg(\mathbbm{1}(\|\bm{u}_{\tilde{k}}^{(i)}\| \leq \hat{C}E)\Big(2 - \frac{2 \gamma}{\hat{C} K \rho} -
    \frac{4 \gamma^2}{\hat{C}^2 K^2 {\rho^2}} \Big)(f_i(\bm{w}_{\tilde{k}}) - f_i(\bm{w}^{*})) + \mathbbm{1}(\|\bm{u}_{\tilde{k}}^{(i)}\| > \hat{C}E) \Big(\frac{3 \hat{C} \|\bm{u}_{\tilde{k}}^{(i)}\|}{8 L E}\Big)\Bigg)\Bigg]
    \\
    \leq \hat{C} \Big(\frac{L \|\bm{w}_{0} - \bm{w}^{*}\|^2}{\gamma} + \frac{\gamma}{L}\Big)\rho + \Bigg(\frac{2 \gamma}{\hat{C} K \rho}\Bigg) \Bigg(1 + \frac{2 \gamma}{\hat{C} K \rho}\Bigg)
    \mathbb{E}\Bigg[\frac{1}{n}\sum_{i=1}^n {\mathbbm{1}(\|\bm{u}_{\tilde{k}}^{(i)}\| \leq \hat{C}E)} \Delta_i^{*}\Bigg],
\end{multline*}
\normalsize
with $\tilde{k} \sim \textup{unif}[0,K-1]$.
\\
Specifically, with $K = \big(\frac{2 \alpha \gamma}{\hat{C} E}\big) \frac{1}{\rho^2}$ and $E \leq \frac{\alpha}{2 \rho}$, where $\alpha \geq 1$ is another constant of our choice, \Cref{alg:dp-fedavg} has the following convergence guarantee:
\begin{multline*}
        \mathbb{E}\Bigg[\frac{1}{n}\sum_{i=1}^n \Bigg(\mathbbm{1}(\|\bm{u}_{\tilde{k}}^{(i)}\| \leq  \hat{C}E) \Big(2 - \frac{\rho E}{\alpha} - \frac{\rho^2 E^2}{\alpha^2}\Big)(f_i(\bm{w}_{\tilde{k}}) - f_i(\bm{w}^{*})) +  \mathbbm{1}(\|\bm{u}_{\tilde{k}}^{(i)}\| > \hat{C}E) \Big(\frac{3 \hat{C} \|\bm{u}_{\tilde{k}}^{(i)}\|}{8 L E}\Big)\Bigg)\Bigg]
        \\
        \leq \hat{C}\Big(\frac{L \|\bm{w}_{0} - \bm{w}^{*}\|^2}{\gamma} + \frac{\gamma}{L}\Big)\rho + \Big(\frac{3 E}{2 \alpha}\Big) \mathbb{E}\Bigg[\frac{1}{n}\sum_{i=1}^n {\mathbbm{1}(\|\bm{u}_{\tilde{k}}^{(i)}\| \leq \hat{C}E)} \Delta_i^{*}\Bigg] \rho.
\end{multline*}
\end{theorem}

\subsection{Proof of Theorem~\ref{thm-clip-cvx}:}
\begin{proof}
Let us set $\eta_k = \beta_k = \eta$ for all $k \geq 0$.
\\
\\
\noindent The update rule of the global iterate is:
\begin{equation}
    \label{jan30-2}
    \bm{w}_{k+1} = \bm{w}_k - \eta \Big(\frac{1}{r}\sum_{i \in \mathcal{S}_k} \text{clip}\big(\bm{u}_k^{(i)}, C \big) + \bm{\zeta}_k\Big),
\end{equation}
where $\bm{\zeta}_k = \frac{1}{r} \sum_{i \in \mathcal{S}_k} \bm{\zeta}^{(i)}_k \sim \mathcal{N}(\vec{0}_d, \frac{q K \log(1/\delta)C^2}{n^2 \varepsilon^2} \bm{\textup{I}}_d)$ and 
\begin{equation}
    \label{jan30-1}
    \bm{u}_k^{(i)} = \frac{(\bm{w}_k - \bm{w}_{k,E}^{(i)})}{\eta} = \sum_{\tau=0}^{E-1}\nabla f_i(\bm{w}_{k,\tau}^{(i)}).
\end{equation}
Taking expectation with respect to the randomness in the current round, we get for any $\bm{w}^{*} \in \arg \min_{\bm{w}' \in \mathbb{R}^d} f(\bm{w}')$:
\small
\begin{flalign}
    \mathbb{E}[\|\bm{w}_{k+1} - \bm{w}^{*}\|^2] & = \mathbb{E}\Big[\Big\|\bm{w}_{k} - \eta\Big(\frac{1}{r}\sum_{i \in \mathcal{S}_k}\text{clip}(\bm{u}_k^{(i)}, C) + \bm{\zeta}_k \Big) - \bm{w}^{*}\Big\|^2\Big] 
    \\
    & = \|\bm{w}_{k} - \bm{w}^{*}\|^2 - 2 \eta \mathbb{E}_{\mathcal{S}_k}\Big[ \frac{1}{r}\sum_{i \in \mathcal{S}_k}\langle \text{clip}(\bm{u}_k^{(i)}, C), \bm{w}_{k} - \bm{w}^{*}\rangle\Big] + \eta^2 \mathbb{E}\Big[\Big\|\frac{1}{r}\sum_{i \in \mathcal{S}_k}\text{clip}(\bm{u}_k^{(i)}, C) + \bm{\zeta}_k\Big\|^2\Big]
    \\
    & = 
    \|\bm{w}_{k} - \bm{w}^{*}\|^2 + \frac{1}{n}\sum_{i = 1}^n {- 2 \eta\langle \text{clip}(\bm{u}_k^{(i)}, C), \bm{w}_{k} - \bm{w}^{*}\rangle}
    + \eta^2 \mathbb{E}_{\mathcal{S}_k}\Big[\Big\|\frac{1}{r}\sum_{i \in \mathcal{S}_k}\text{clip}(\bm{u}_k^{(i)}, C)\Big\|^2\Big]
    \\
    \nonumber
    & \hspace{1 cm}
    + \eta^2 \Big(\frac{q K d \log(1/\delta) C^2}{n^2 \varepsilon^2}\Big)
    \\
    \label{eq:feb4-10}
    & \leq \|\bm{w}_{k} - \bm{w}^{*}\|^2 + \frac{1}{n}\sum_{i = 1}^n {- 2\eta\langle \text{clip}(\bm{u}_k^{(i)}, C), \bm{w}_{k} - \bm{w}^{*}\rangle}
    + \eta^2 \mathbb{E}_{\mathcal{S}_k}\Big[\frac{1}{r}\sum_{i \in \mathcal{S}_k}\big\|\text{clip}(\bm{u}_k^{(i)}, C)\big\|^2\Big]
    \\
    \nonumber
    & \hspace{1 cm}
    + \eta^2 \Big(\frac{q K d \log(1/\delta) C^2}{n^2 \varepsilon^2}\Big)
    \\
    \label{eq:pf-july15-2}
    & = 
    \|\bm{w}_{k} - \bm{w}^{*}\|^2 + \frac{1}{n}\sum_{i = 1}^n \Big\{\underbrace{- 2 \eta\langle \text{clip}(\bm{u}_k^{(i)}, C), \bm{w}_{k} - \bm{w}^{*}\rangle + \eta^2 \big\|\text{clip}(\bm{u}_k^{(i)}, C)\big\|^2}_{A_i}\Big\}
    \\
    \nonumber
    & \hspace{1 cm}
    + \eta^2 \Big(\frac{q K d \log(1/\delta) C^2}{n^2 \varepsilon^2}\Big).
\end{flalign}
\normalsize
Note that \cref{eq:feb4-10} is obtained by using \Cref{fact2}. 
Let us examine $A_i$ for each $i$.
\\
\\
\textbf{Case 1:} $\|\bm{u}_k^{(i)}\| > C$. So we have $\text{clip}(\bm{u}_k^{(i)}, C) = \frac{C}{\|\bm{u}_k^{(i)}\|} \bm{u}_k^{(i)}$. Thus,
\begin{flalign}
    \label{eq:pf-july15-3}
    A_i & = {-2 \eta C} {\Big\langle \frac{\bm{u}_k^{(i)}}{\|\bm{u}_k^{(i)}\|}, \bm{w}_{k} - \bm{w}^{*} \Big\rangle} + \eta^2 C^2
    \\
    \label{eq:jan30-2-1}
    & = \frac{-C}{ \big\|\bm{u}_k^{(i)}\big\|} \Big(\|\bm{w}_{k} - \bm{w}^{*}\|^2 + \eta^2 \big\|\bm{u}_k^{(i)}\big\|^2 - \|\underbrace{\bm{w}_{k} - \eta \bm{u}_k^{(i)}}_{=\bm{w}_{k,E}^{(i)}} - \bm{w}^{*}\|^2\Big) + \eta^2 C^2,
\end{flalign}
where the last step follows by using the fact for any two vectors $\bm{a}$ and $\bm{b}$, $\langle \bm{a}, \bm{b} \rangle = \frac{1}{2}\Big(\|\bm{a}\|^2 + \|\bm{b}\|^2 - \|\bm{a}-\bm{b}\|^2)$. Next, notice that $\bm{w}_{k} - \eta \bm{u}_k^{(i)} = \bm{w}_{k,E}^{(i)}$.
Since $f_i$ is convex, we use \Cref{lem-mar15-1} to get:
\begin{equation}
    \label{jan30-3}
    \|\bm{w}_{k} - \bm{w}^{*}\|^2 - \|\bm{w}_{k,E}^{(i)} - \bm{w}^{*}\|^2 \geq \frac{\eta}{2L} \sum_{\tau=0}^{E-1}\|\nabla f_i(\bm{w}_{k,\tau}^{(i)})\|^2 - 2 \eta E  \Delta_i^{*},
\end{equation}
for $\eta \leq \frac{1}{2L}$ with $\Delta_i^{*} := f_i(\bm{w}^{*}) - \min_{\bm{w}' \in \mathbb{R}^d} f_i(\bm{w}') \geq 0$. But:
\begin{flalign}
    \label{mar15-5}
    \big\|\bm{u}_k^{(i)}\big\|^2 & = \big\|\sum_{\tau=0}^{E-1}\nabla f_i(\bm{w}_{k,\tau}^{(i)})\big\|^2 \leq E\sum_{\tau=0}^{E-1} \big\|\nabla f_i(\bm{w}_{k,\tau}^{(i)})\big\|^2.
\end{flalign}
The inequality above follows from \Cref{fact2}. Using this in \cref{jan30-3}, we get:
\begin{equation}
    \label{jan30-4}
    \|\bm{w}_{k} - \bm{w}^{*}\|^2 - \|\bm{w}_{k,E}^{(i)} - \bm{w}^{*}\|^2 \geq \frac{\eta}{2LE} \big\|\bm{u}_k^{(i)}\big\|^2 - 2 \eta E  \Delta_i^{*}.
\end{equation}
Plugging this back in \cref{eq:jan30-2-1}, we get:
\begin{equation}
    \label{eq:feb1-0}
    A_i \leq - C \Big(\eta^2 + \frac{\eta}{2 L E}\Big){\|\bm{u}_k^{(i)}\|} + 2 \eta \underbrace{\Bigg(\frac{C}{\|\bm{u}_k^{(i)}\|}\Bigg)}_{<1} E \Delta_i^{*} + \eta^2 C^2,
\end{equation}
for $\eta \leq \frac{1}{2L}$.

Let us choose $C^2 \geq 16 L E^2 \max_{j \in [n]} \Delta_j^{*}$. Then, we have $E \Delta_i^{*} \leq \frac{C^2}{16 L E} \leq \frac{C \|\bm{u}_k^{(i)}\|}{16 L E}$. Using this in \cref{eq:feb1-0}, we get:
\begin{flalign}
    A_i & \leq - C \Big(\eta^2 + \frac{\eta}{2 L E}\Big)\|\bm{u}_k^{(i)}\| + \frac{\eta C}{8 L E} \|\bm{u}_k^{(i)}\| + \eta^2 C^2
    \\
    & = -\frac{3 \eta C}{8 L E} \|\bm{u}_k^{(i)}\| + \eta^2 C \underbrace{(C - \|\bm{u}_k^{(i)}\|)}_{< 0}
    \\
    \label{eq:pf-july15-5-1}
    & \leq -\frac{3 \eta C}{8 L E} \|\bm{u}_k^{(i)}\|,
\end{flalign}
for $C \geq 4 E \sqrt{L \max_{j \in [n]} \Delta_j^{*}}$ and $\eta \leq \frac{1}{2 L}$.
\\
\\
\textbf{Case 2:} $\|\bm{u}_k^{(i)}\| \leq C$. So we have $\text{clip}(\bm{u}_k^{(i)}, C) = \bm{u}_k^{(i)}$. Thus,
\begin{flalign}
    \label{eq:pf-july15-6}
    A_i & = {-2 \eta} {\langle \bm{u}_k^{(i)}, \bm{w}_{k} - \bm{w}^{*} \rangle} + \eta^2 \|\bm{u}_k^{(i)}\|^2 \leq {-2 \eta} \underbrace{\langle \bm{u}_k^{(i)}, \bm{w}_{k} - \bm{w}^{*} \rangle}_{B_i} + 2 \eta^2 L E^2 (f_i(\bm{w}_k) - f_i^{*}),
\end{flalign}
for $\eta L \leq {1}$; the inequality $\|\bm{u}_{{k}}^{(i)}\|^2 \leq 2 L E^2 (f_i(\bm{w}_k) - f_i^{*})$ (for $\eta L \leq {1}$) is obtained from \Cref{lem-mar13-1}. 
Now:
\begin{flalign}
    B_i & = \langle \bm{u}_k^{(i)}, \bm{w}_{k} - \bm{w}^{*} \rangle
    \\
     \label{eq:pf-july15-7}
    & = \sum_{\tau=0}^{E-1} \langle \nabla f_i(\bm{w}_{k,\tau}^{(i)}), \bm{w}_{k} - \bm{w}^{*} \rangle
    \\
     \label{eq:pf-july15-8}
    & = \sum_{\tau=0}^{E-1} \{\langle \nabla f_i(\bm{w}_{k,\tau}^{(i)}), \bm{w}_{k,\tau}^{(i)} - \bm{w}^{*} \rangle + \langle \nabla f_i(\bm{w}_{k,\tau}^{(i)}), \bm{w}_{k} - \bm{w}_{k,\tau}^{(i)} \rangle \}
    \\
     \label{eq:pf-july15-9}
    & \geq \sum_{\tau=0}^{E-1} \{f_i(\bm{w}_{k,\tau}^{(i)}) - f_i(\bm{w}^{*}) + \langle \nabla f_i(\bm{w}_{k}), \bm{w}_{k} - \bm{w}_{k,\tau}^{(i)} \rangle + \langle \nabla f_i(\bm{w}_{k,\tau}^{(i)}) - \nabla f_i(\bm{w}_{k}), \bm{w}_{k} - \bm{w}_{k,\tau}^{(i)} \rangle \}
    \\
     \label{eq:pf-july15-10}
    & \geq \sum_{\tau=0}^{E-1} \{f_i(\bm{w}_{k,\tau}^{(i)}) - f_i(\bm{w}^{*}) + f_i(\bm{w}_{k}) - f_i(\bm{w}_{k,\tau}^{(i)}) - L\|\bm{w}_{k} - \bm{w}_{k,\tau}^{(i)}\|^2\}
    \\
     \label{eq:pf-july15-11}
    & = E (f_i(\bm{w}_{k}) - f_i(\bm{w}^{*})) - L \sum_{\tau=0}^{E-1}\|\bm{w}_{k} - \bm{w}_{k,\tau}^{(i)}\|^2.
\end{flalign}
Note that \cref{eq:pf-july15-9} follows from the convexity of $f_i$, while \cref{eq:pf-july15-10} follows by once again using the convexity of $f_i$, the smoothness of $f_i$ as well as the Cauchy-Schwarz inequality.
\\
Again, from \Cref{lem-mar13-1}, we have
\begin{equation}
     \label{eq:pf-july15-12}
     \|\bm{w}_{k} - \bm{w}_{k,\tau}^{(i)}\|^2 \leq 2 \eta^2 L \tau^2 (f_i(\bm{w}_{k}) - f_i^{*}),
\end{equation}
for $\eta L \leq 1$. Using \cref{eq:pf-july15-12} in \cref{eq:pf-july15-11}, we get
\begin{equation}
    \label{eq:pf-july15-13}
    B_i \geq E (f_i(\bm{w}_{k}) - f_i(\bm{w}^{*})) - 2 \eta^2 L^2 \sum_{\tau=0}^{E-1} \tau^2 (f_i(\bm{w}_{k}) - f_i^{*}) \geq E (f_i(\bm{w}_{k}) - f_i(\bm{w}^{*})) - 2 \eta^2 L^2 E^3 (f_i(\bm{w}_{k}) - f_i^{*}).
\end{equation}
Now using \cref{eq:pf-july15-13} in \cref{eq:pf-july15-6}, we get
\begin{flalign}
    \nonumber
    A_i & \leq - 2 \eta E (f_i(\bm{w}_{k}) - f_i(\bm{w}^{*})) + 4 \eta^3 L^2 E^3 (f_i(\bm{w}_{k}) - f_i^{*}) + 2 \eta^2 L E^2 (f_i(\bm{w}_{k}) - f_i^{*})
    \\
    \label{eq:pf-july15-16}
    & = -\eta E {\Big(2 - 2 \eta L E - 4 \eta^2 L^2 E^2\Big)} (f_i(\bm{w}_{k}) - f_i(\bm{w}^{*})) + (2 \eta^2 L E^2 + 4 \eta^3 L^2 E^3) \Delta_i^{*},
\end{flalign}
for $\eta \leq \frac{1}{L}$.
\\
\\
Combining the results of Case 1 and 2, i.e. \cref{eq:pf-july15-5-1} and \cref{eq:pf-july15-16}, we get
\small
\begin{multline}
    \label{eq:pf-july15-18}
    A_i \leq \mathbbm{1}(\|\bm{u}_k^{(i)}\| > C) \Big(-\frac{3 \eta C}{8 L E} \|\bm{u}_k^{(i)}\|\Big) 
    \\
    + \mathbbm{1}(\|\bm{u}_k^{(i)}\| \leq C) \Big(-{\eta E} (2 - 2 \eta L E - 4 \eta^2 L^2 E^2) (f_i(\bm{w}_{k}) - f_i(\bm{w}^{*})) + (2 \eta^2 L E^2 + 4 \eta^3 L^2 E^3) \Delta_i^{*}\Big),
\end{multline}
\normalsize
for $C \geq 4 E \sqrt{L \max_{j \in [n]} \Delta_j^{*}}$ and $\eta \leq \frac{1}{2 L}$.
Let us define $\hat{C} := \frac{C}{E}$. Then \cref{eq:pf-july15-18} can be re-written as:
\small
\begin{multline}
    \label{eq:pf-july15-19-0}
    A_i \leq 
    -{\eta E}\Bigg\{\mathbbm{1}(\|\bm{u}_k^{(i)}\| \leq \hat{C} E) \Big((2 - 2 \eta L E - 4 \eta^2 L^2 E^2)(f_i(\bm{w}_{k}) - f_i(\bm{w}^{*}))\Big)
    \\
    + {\mathbbm{1}(\|\bm{u}_k^{(i)}\| > \hat{C} E)\Big(\frac{3 \hat{C}}{8 L E}\|\bm{u}_k^{(i)}\|\Big)} - \mathbbm{1}(\|\bm{u}_k^{(i)}\| \leq \hat{C} E) (2 \eta L E + 4 \eta^2 L^2 E^2) \Delta_i^{*} \Bigg\},
\end{multline}
\normalsize
where $\hat{C} \geq 4 \sqrt{L \max_{j \in [n]} \Delta_j^{*}}$ and $\eta \leq \frac{1}{2 L}$. Now using \cref{eq:pf-july15-19-0} in \cref{eq:pf-july15-2}, we get:
\small
\begin{multline}
    \label{eq:pf-july15-20}
    \mathbb{E}[\|\bm{w}_{k+1} - \bm{w}^{*}\|^2] \leq \|\bm{w}_{k} - \bm{w}^{*}\|^2 - \frac{\eta E}{n}\sum_{i=1}^n \Big\{\mathbbm{1}(\|\bm{u}_{{k}}^{(i)}\| > \hat{C}E) \Big(\frac{3 \hat{C}}{8 L E}\|\bm{u}_k^{(i)}\|\Big)
    \\
    + 
    \mathbbm{1}(\|\bm{u}_{{k}}^{(i)}\| \leq \hat{C}E) \Big((2 - 2 \eta L E - 4 \eta^2 L^2 E^2)(f_i(\bm{w}_{{k}}) - f_i(\bm{w}^{*}))\Big)\Big\}
    \\
    + \eta E (2 \eta L E + 4 \eta^2 L^2 E^2) \Big(\frac{1}{n}\sum_{i=1}^n {\mathbbm{1}(\|\bm{u}_{{k}}^{(i)}\| \leq \hat{C}E)} \Delta_i^{*}\Big)
    + \eta^2 E^2 \hat{C}^2 \Big(\frac{q K d \log(1/\delta)}{n^2 \varepsilon^2}\Big).
\end{multline}
\normalsize
Solving the above recursion after taking expectation throughout and some rearranging, we get:
\small
\begin{multline}
    \label{eq:pf-july15-21}
    \frac{1}{K}\sum_{k=0}^{K-1}
    \mathbb{E}\Big[\frac{1}{n}\sum_{i=1}^n \Big\{\mathbbm{1}(\|\bm{u}_{{k}}^{(i)}\| \leq \hat{C}E) \Big((2 - 2 \eta L E - 4 \eta^2 L^2 E^2)(f_i(\bm{w}_{{k}}) - f_i(\bm{w}^{*}))\Big) + \mathbbm{1}(\|\bm{u}_{{k}}^{(i)}\| > \hat{C}E) \Big(\frac{3 \hat{C} \|\bm{u}_k^{(i)}\|}{8 L E}\Big)\Big\}\Big] 
    \\
    \leq \frac{\|\bm{w}_{0} - \bm{w}^{*}\|^2}{\eta E K}
    + \eta E K \hat{C}^2 \Big(\frac{q d \log(1/\delta)}{n^2 \varepsilon^2}\Big)
    + \frac{2 \eta L E (1 + 2 \eta L E)}{K}\mathbb{E}\Bigg[\sum_{k=0}^{K-1} \Big(\frac{1}{n}\sum_{i=1}^n {\mathbbm{1}(\|\bm{u}_{{k}}^{(i)}\| \leq \hat{C}E)} \Delta_i^{*}\Big)\Bigg].
\end{multline}
\normalsize
Let us choose $\eta = \frac{\gamma}{\hat{C} L E K} \frac{n \varepsilon}{\sqrt{q d \log(1/\delta)}}$ for some constant $\gamma > 0$. Note that we must have $K > \frac{2 \gamma}{\hat{C} E} \frac{n \varepsilon}{\sqrt{q d \log(1/\delta)}}$ for our condition of $\eta L \leq \frac{1}{2}$ to be satisfied. With that, we get:
\small
\begin{multline}
    \frac{1}{K}\sum_{k=0}^{K-1}
    \mathbb{E}\Bigg[\frac{1}{n}\sum_{i=1}^n \Bigg\{\mathbbm{1}(\|\bm{u}_{{k}}^{(i)}\| \leq \hat{C}E) \Big(2 - \frac{2 \gamma}{\hat{C} K} \frac{n \varepsilon}{\sqrt{q d \log(1/\delta)}} -
    \frac{4 \gamma^2}{\hat{C}^2 K^2} \frac{n^2 \varepsilon^2}{q d \log(1/\delta)}\Big)(f_i(\bm{w}_{{k}}) - f_i(\bm{w}^{*})) 
    \\
    + \mathbbm{1}(\|\bm{u}_{{k}}^{(i)}\| > \hat{C}E)\Big(\frac{3 \hat{C}}{8 L E}\|\bm{u}_k^{(i)}\|\Big)\Bigg\}\Bigg] \leq 
    \Big(\frac{L \|\bm{w}_{0} - \bm{w}^{*}\|^2}{\gamma}
    + \frac{\gamma}{L}\Big)\frac{\hat{C} \sqrt{q d \log(1/\delta)}}{n \varepsilon}
    \\
    + \Bigg(\frac{2 \gamma}{\hat{C} K} \frac{n \varepsilon}{\sqrt{q d \log(1/\delta)}}\Bigg) \Bigg(1 + \frac{2 \gamma}{\hat{C} K} \frac{n \varepsilon}{\sqrt{q d \log(1/\delta)}}\Bigg)\mathbb{E}\Bigg[\Bigg(\frac{1}{K}\sum_{k=0}^{K-1}\Big(\frac{1}{n}\sum_{i=1}^n {\mathbbm{1}(\|\bm{u}_{{k}}^{(i)}\| \leq \hat{C}E)} \Delta_i^{*}\Big)\Bigg)\Bigg],
\end{multline}
\normalsize
with $\hat{C} \geq 4 \sqrt{L \max_{j \in [n]} \Delta_j^{*}}$.
\\
\\
The above equation is equivalent to:
\small
\begin{multline}
    \label{eq:jan22-3}
     \mathbb{E}\Bigg[\frac{1}{n}\sum_{i=1}^n \Bigg\{\mathbbm{1}(\|\bm{u}_{\tilde{k}}^{(i)}\| \leq \hat{C}E) \Big(2 - \frac{2 \gamma}{\hat{C} K} \frac{n \varepsilon}{\sqrt{q d \log(1/\delta)}} -
    \frac{4 \gamma^2}{\hat{C}^2 K^2} \frac{n^2 \varepsilon^2}{q d \log(1/\delta)}\Big)(f_i(\bm{w}_{\tilde{k}}) - f_i(\bm{w}^{*})) 
    \\
    + \mathbbm{1}(\|\bm{u}_{\tilde{k}}^{(i)}\| > \hat{C}E) \Big(\frac{3 \hat{C}}{8 L E}\|\bm{u}_{\tilde{k}}^{(i)}\|\Big) \Bigg\}\Bigg] \leq
    \Big(\frac{L \|\bm{w}_{0} - \bm{w}^{*}\|^2}{\gamma} + \frac{\gamma}{L}\Big)\frac{\hat{C} \sqrt{q d \log(1/\delta)}}{n \varepsilon} 
    \\
    + 
    \Bigg(\frac{2 \gamma}{\hat{C} K} \frac{n \varepsilon}{\sqrt{q d \log(1/\delta)}}\Bigg) \Bigg(1 + \frac{2 \gamma}{\hat{C} K} \frac{n \varepsilon}{\sqrt{q d \log(1/\delta)}}\Bigg)
    \mathbb{E}\Bigg[\frac{1}{n}\sum_{i=1}^n {\mathbbm{1}(\|\bm{u}_{\tilde{k}}^{(i)}\| \leq \hat{C}E)} \Delta_i^{*}\Bigg],
\end{multline}
\normalsize
where $\tilde{k} \sim \text{unif}[0,K-1]$. Let us set $K = \frac{2 \alpha \gamma}{\hat{C} E} \Big(\frac{n \varepsilon}{\sqrt{q d \log(1/\delta)}}\Big)^2$ in \cref{eq:jan22-3}, where $\alpha \geq 1$ is a constant of our choice and $E \leq \frac{\alpha}{2}\Big(\frac{n \varepsilon}{\sqrt{q d \log(1/\delta)}}\Big)$. That gives us:
\small
\begin{multline}
    \label{eq:jan30-5}
     \mathbb{E}\Bigg[\frac{1}{n}\sum_{i=1}^n \Bigg\{\mathbbm{1}(\|\bm{u}_{\tilde{k}}^{(i)}\| \leq \hat{C}E) \Bigg(2 - \Bigg(\frac{E\sqrt{q d \log(1/\delta)}}{\alpha n \varepsilon}\Bigg) -  \Bigg(\frac{E\sqrt{q d \log(1/\delta)}}{\alpha n \varepsilon}\Bigg)^2
     \Bigg)(f_i(\bm{w}_{\tilde{k}}) - f_i(\bm{w}^{*})) 
     \\
     + \mathbbm{1}(\|\bm{u}_{\tilde{k}}^{(i)}\| > \hat{C}E) \Big(\frac{3 \hat{C}}{8 L E}\|\bm{u}_{\tilde{k}}^{(i)}\|\Big) \Bigg\}\Bigg]
     \leq
     \hat{C}\Big(\frac{L \|\bm{w}_{0} - \bm{w}^{*}\|^2}{\gamma} + \frac{\gamma}{L}\Big)\frac{\sqrt{q d \log(1/\delta)}}{n \varepsilon} 
     \\
     + \mathbb{E}\Bigg[\frac{1}{n}\sum_{i=1}^n {\mathbbm{1}(\|\bm{u}_{\tilde{k}}^{(i)}\| \leq \hat{C}E)} \Delta_i^{*}\Bigg]
     \frac{E\sqrt{q d \log(1/\delta)}}{\alpha n \varepsilon} \underbrace{\Bigg(1 + \frac{E\sqrt{q d \log(1/\delta)}}{\alpha n \varepsilon}\Bigg)}_{\leq \frac{3}{2} \text{ from our constraint on $E$}}.
\end{multline}
\normalsize
The final result follows by substituting $\rho = \frac{\sqrt{q d \log(1/\delta)}}{n \varepsilon}$.
\end{proof}

\section{Proof of Theorem~\ref{cor-1}:}
\label{cor-1-pf}
\begin{proof}
First, using
\begin{multline*}
    \min\Bigg(\Big(2 - \frac{\rho E}{\alpha} - \frac{\rho^2 E^2}{\alpha^2}\Big)(f_i(\bm{w}_{\tilde{k}}) - f_i(\bm{w}^{*})), \frac{3 \hat{C}}{8 L E} \|\bm{u}_{\tilde{k}}^{(i)}\|\Bigg) 
    \\
    \leq \Bigg(\mathbbm{1}(\|\bm{u}_{\tilde{k}}^{(i)}\| \leq \hat{C}E) \Big(2 - \frac{\rho E}{\alpha} - \frac{\rho^2 E^2}{\alpha^2}\Big)(f_i(\bm{w}_{\tilde{k}}) - f_i(\bm{w}^{*})) +  \mathbbm{1}(\|\bm{u}_{\tilde{k}}^{(i)}\| > \hat{C}E) \Big(\frac{3 \hat{C}}{8 L E} \|\bm{u}_{\tilde{k}}^{(i)}\|\Big)\Bigg),
\end{multline*}
$\mathbb{E}\Big[\frac{1}{n}\sum_{i=1}^n {\mathbbm{1}(\|\bm{u}_{\tilde{k}}^{(i)}\| \leq \hat{C}E)} \Delta_i^{*}\Big] \leq \frac{1}{n}\sum_{i=1}^n \Delta_i^{*}$, $\Big(2 - \frac{\rho E}{\alpha} - \frac{\rho^2 E^2}{\alpha^2}\Big) = \mathcal{O}(1)$ as $E \leq \frac{\alpha}{2 \rho}$, and plugging in $\gamma = \mathcal{O}(L \|\bm{w}_0 - \bm{w}^{*}\|)$ and $\alpha = \mathcal{O}(1)$ in \Cref{thm-clip-cvx-short}, we get:
\begin{equation}
    \label{eq:cor-2.1-1}
     \mathbb{E}\Bigg[\frac{1}{n}\sum_{i=1}^n \min\Bigg(f_i(\bm{w}_{\tilde{k}}) - f_i(\bm{w}^{*}), \mathcal{O}\Big(\frac{\hat{C}}{L E} \|\bm{u}_{\tilde{k}}^{(i)}\|\Big)\Bigg)\Bigg]
     \leq \mathcal{O}\Bigg(\hat{C} \|\bm{w}_{0} - \bm{w}^{*}\| + E \Big(\frac{1}{n}\sum_{i=1}^n \Delta_i^{*}\Big)\Bigg)\rho.
\end{equation}
Now we need to lower bound $\|\bm{u}_{\tilde{k}}^{(i)}\|$ in terms of $\|\nabla f_i(\bm{w}_{\tilde{k}})\|$. To that end, note that:
\begin{flalign}
    \|\bm{u}_{\tilde{k}}^{(i)}\|^2 &= \Big\|\sum_{\tau=0}^{E-1} \nabla f_i(\bm{w}_{\tilde{k},\tau}^{(i)})\Big\|^2
    \\
    \label{eq:cor-2.1-2}
    & = \sum_{\tau,\tau'} \frac{1}{2}\Big(\|\nabla f_i(\bm{w}_{\tilde{k},\tau}^{(i)})\|^2 + \|\nabla f_i(\bm{w}_{\tilde{k},\tau'}^{(i)})\|^2 - \|\nabla f_i(\bm{w}_{\tilde{k},\tau}^{(i)}) - \nabla f_i(\bm{w}_{\tilde{k},\tau'}^{(i)})\|^2\Big)
    \\
    \label{eq:cor-2.1-2-a}
    & = E \sum_{\tau=0}^{E-1} \|\nabla f_i(\bm{w}_{\tilde{k},\tau}^{(i)})\|^2 - \sum_{\tau < \tau'} \|\nabla f_i(\bm{w}_{\tilde{k},\tau}^{(i)}) - \nabla f_i(\bm{w}_{\tilde{k},\tau'}^{(i)})\|^2.
\end{flalign}
\Cref{eq:cor-2.1-2} follows from the fact that for any two vectors $\bm{a}$ and $\bm{b}$, $\langle \bm{a}, \bm{b} \rangle = \frac{1}{2}(\|\bm{a}\|^2 + \|\bm{b}\|^2 - \|\bm{a} - \bm{b}\|^2)$. Next, by using the $L$-smoothness of $f_i$, we have for $\tau < \tau'$:
\begin{flalign}
    \|\nabla f_i(\bm{w}_{\tilde{k},\tau}^{(i)}) - \nabla f_i(\bm{w}_{\tilde{k},\tau'}^{(i)})\| & \leq L \|\bm{w}_{\tilde{k},\tau}^{(i)} - \bm{w}_{\tilde{k},\tau'}^{(i)}\|
    \\
    & = \eta L \Big\|\sum_{t=\tau}^{\tau'-1} \nabla f_i(\bm{w}_{\tilde{k},t}^{(i)})\Big\|
    \\
    \label{eq:cor-2.1-3}
    & \leq \eta L \sum_{t=\tau}^{\tau'-1} \|\nabla f_i(\bm{w}_{\tilde{k},t}^{(i)})\|.
\end{flalign}
But from \Cref{local_steps_lem}, we have that $\|\nabla f_i(\bm{w}_{\tilde{k},t}^{(i)})\| \leq \|\nabla f_i(\bm{w}_{\tilde{k},{t-1}}^{(i)})\| \leq \ldots \leq \|\nabla f_i(\bm{w}_{\tilde{k},0}^{(i)})\| = \|\nabla f_i(\bm{w}_{\tilde{k}})\|$. Using this in \cref{eq:cor-2.1-3}, we get:
\begin{equation}
    \|\nabla f_i(\bm{w}_{\tilde{k},\tau}^{(i)}) - \nabla f_i(\bm{w}_{\tilde{k},\tau'}^{(i)})\| \leq \eta L (\tau' - \tau) \|\nabla f_i(\bm{w}_{\tilde{k}})\| \leq \eta L E \|\nabla f_i(\bm{w}_{\tilde{k}})\|.
\end{equation}
Plugging this into \cref{eq:cor-2.1-2-a}, we get:
\begin{flalign}
    \|\bm{u}_{\tilde{k}}^{(i)}\|^2 & \geq E \sum_{\tau=0}^{E-1} \|\nabla f_i(\bm{w}_{\tilde{k},\tau}^{(i)})\|^2 - \sum_{\tau < \tau'} \eta^2 L^2 E^2 \|\nabla f_i(\bm{w}_{\tilde{k}})\|^2 
    \\
    \label{eq:cor-2.1-4}
    & \geq E \sum_{\tau=0}^{E-1} \|\nabla f_i(\bm{w}_{\tilde{k},\tau}^{(i)})\|^2 - \frac{\eta^2 L^2 E^4}{2} \|\nabla f_i(\bm{w}_{\tilde{k}})\|^2.
\end{flalign}
Further, for any $\tau \geq 1$:
\begin{flalign}
    \|\nabla f_i(\bm{w}_{\tilde{k}})\| & \leq \|\nabla f_i(\bm{w}_{\tilde{k},\tau}^{(i)})\| + \|\nabla f_i(\bm{w}_{\tilde{k}}) - \nabla f_i(\bm{w}_{\tilde{k},\tau}^{(i)})\|
    \\
    \label{eq:cor-2.1-5}
    & \leq \|\nabla f_i(\bm{w}_{\tilde{k},\tau}^{(i)})\| + L {\|\bm{w}_{\tilde{k}} - \bm{w}_{\tilde{k},\tau}^{(i)}\|}.
\end{flalign}
Recall that $\eta = \frac{\rho}{2 \alpha L}$ and $E \leq \frac{\alpha}{2\rho}$ in \Cref{thm-clip-cvx-short}, due to which
$\eta L E \leq \frac{1}{4}$. Thus, we can apply \Cref{lem-apr13-1} in \cref{eq:cor-2.1-5} to obtain:
\begin{equation}
    \|\nabla f_i(\bm{w}_{\tilde{k}})\| \leq \|\nabla f_i(\bm{w}_{\tilde{k},\tau}^{(i)})\| + 2 \eta L \tau \|\nabla f_i(\bm{w}_{\tilde{k}})\|.
\end{equation}
Now using the fact that $\eta L \tau \leq \eta L E \leq \frac{1}{4}$ above, we get:
\begin{equation}
    \label{eq:cor-2.1-6}
    \|\nabla f_i(\bm{w}_{\tilde{k},\tau}^{(i)})\| \geq \frac{\|\nabla f_i(\bm{w}_{\tilde{k}})\|}{2} \text{ } \forall \text{ } \tau \geq 1.
\end{equation}
Plugging this back in \cref{eq:cor-2.1-4} and using the fact that $\eta L E \leq \frac{1}{4}$, we get:
\begin{equation}
    \|\bm{u}_{\tilde{k}}^{(i)}\|^2 \geq \frac{E^2}{4}\big(1 - 2\eta^2 L^2 E^2\big) \|\nabla f_i(\bm{w}_{\tilde{k}})\|^2 \geq \frac{7 E^2}{32} \|\nabla f_i(\bm{w}_{\tilde{k}})\|^2.
\end{equation}
So, we have:
\begin{equation}
    \|\bm{u}_{\tilde{k}}^{(i)}\| \geq \mathcal{O}(E \|\nabla f_i(\bm{w}_{\tilde{k}})\|).
\end{equation}
Using this in \cref{eq:cor-2.1-1} gives us the final result.
\end{proof}

\section{Proof of Proposition~\ref{local_steps}}
\label{local_steps_pf}
\begin{proof}
First, note that with $\eta_k = \eta$, we have:
\begin{flalign}
    \|\bm{u}_k^{(i)}\| & = \Bigg\|\frac{\bm{w}_{k} - \bm{w}_{k, E}^{(i)}}{\eta}\Bigg\|
    \\
    & = \Bigg\|\sum_{\tau=0}^{E-1}\nabla f_i(\bm{w}_{k, \tau}^{(i)})\Bigg\|
    \\
    \label{eq:mar12-1003}
    & \leq \sum_{\tau=0}^{E-1} \|\nabla f_i(\bm{w}_{k, \tau}^{(i)})\|.
\end{flalign}
{Now using the result of \Cref{local_steps_lem} and applying our assumption that $\|\nabla f_i(\bm{w}_{k,\tau+1}^{(i)}) - \nabla f_i(\bm{w}_{k,\tau}^{(i)}) \| \geq \eta \lambda \|\nabla f_i(\bm{w}_{k,\tau}^{(i)})\|$ in it, we get:
\begin{equation}
    \label{eq:mar12-1002}
    \|\nabla f_i(\bm{w}_{k, \tau+1}^{(i)})\|^2 \leq \Bigg(1 - \frac{2 \eta \lambda^2}{L} \Big(1 - \frac{\eta L}{2}\Big)\Bigg)\|\nabla f_i(\bm{w}_{k,\tau}^{(i)})\|^2.
\end{equation}
Plugging in $\eta = \frac{\rho}{2L}$ above, we get:
\begin{flalign}
    \|\nabla f_i(\bm{w}_{k, \tau+1}^{(i)})\| & \leq \sqrt{1 - \frac{\lambda^2}{L^2} \rho \Big(1 - \frac{\rho}{4}\Big)}\|\nabla f_i(\bm{w}_{k,\tau}^{(i)})\| 
    \\
    & \leq \Bigg(1 - \frac{\lambda^2}{2 L^2} \rho \Big(1 - \frac{\rho}{4}\Big)\Bigg)\|\nabla f_i(\bm{w}_{k,\tau}^{(i)})\|
    \\
    \label{eq:mar12-1002-1}
    & \leq \Bigg(1 - \frac{3 \lambda^2}{8 L^2} \rho \Bigg)\|\nabla f_i(\bm{w}_{k,\tau}^{(i)})\|.
\end{flalign}
For notational convenience, let $\widehat{\rho} := \frac{3 \lambda^2}{8 L^2} \rho$. Then from \cref{eq:mar12-1002-1}, we get:
\begin{equation}
    \|\nabla f_i(\bm{w}_{k, \tau}^{(i)})\| \leq (1 - \widehat{\rho})^\tau\|\nabla f_i(\bm{w}_{k,0}^{(i)})\|.
\end{equation}
Using this in \cref{eq:mar12-1003}, we get:
\begin{equation}
    \label{eq:mar12-1005}
    \|\bm{u}_k^{(i)}\| \leq \sum_{\tau=0}^{E-1} (1 - \widehat{\rho})^\tau \|\nabla f_i(\bm{w}_{k,0}^{(i)})\| = \Big(\frac{1 - (1 - \widehat{\rho})^E}{\widehat{\rho}}\Big) \|\nabla f_i(\bm{w}_{k,0}^{(i)})\| \leq \underbrace{\Big(\frac{1 - (1 - \widehat{\rho})^E}{\widehat{\rho}}\Big)G}_{\text{B}(E)}.
\end{equation}
Recall that $E \leq \frac{1}{2 \rho}$ due to which we have $E \hat{\rho} \leq \frac{1}{4}$. So using \Cref{new-fact-1} in \cref{eq:mar12-1005}, we get:
\begin{equation}
    \text{B}(E) \leq G E \Big(1 - \frac{11(E-1)\hat{\rho}}{24}\Big) = G E \Bigg(1 - \frac{11(E-1){\rho}}{64}\Big(\frac{\lambda^2}{L^2}\Big)\Bigg).
\end{equation}
So if we set $\hat{C} = \frac{\Big(1 - \frac{11(E-1){\rho}}{64}\big(\frac{\lambda^2}{L^2}\big)\Big)}{E} = G \Big(1 - \frac{11(E-1){\rho}}{64}\big(\frac{\lambda^2}{L^2}\big)\Big)$, then we will have no clipping as $\|\bm{u}_k^{(i)}\| \leq \hat{C} E$ always. 
}
\end{proof}

\section{Full Version of Theorem~\ref{thm-new-clip-cvx-short} and its Proof}
\label{thm2-details}
\begin{theorem}[\textbf{Full version of \Cref{thm-new-clip-cvx-short}}]
\label{thm-new-clip-cvx}
Suppose each $f_i$ is convex and $L$-smooth over $\mathbb{R}^d$. Let $\hat{C} := \frac{C}{E}$, where $C$ is the scaling factor used in \Cref{alg:dp-fedavg-2}. For any $\bm{w}^{*} \in \arg \min_{\bm{w}' \in \mathbb{R}^d} f(\bm{w}')$ and $\Delta_i^{*} := f_i(\bm{w}^{*}) - \min_{\bm{w}' \in \mathbb{R}^d} f_i(\bm{w}') \geq 0$, \Cref{alg:dp-fedavg-2} with $\hat{C} \geq 4 \sqrt{L \max_{j \in [n]} \Delta_j^{*}}$, $\beta_k = \eta_k = \eta = \big(\frac{\gamma}{\hat{C} L E K}\big) \frac{1}{\rho}$ and $K > \big(\frac{2 \gamma}{\hat{C} E}\big) \frac{1}{\rho}$, where $\gamma > 0$ is a constant of our choice, has the following convergence guarantee:
\small
\begin{multline*}
    \mathbb{E}\Bigg[\frac{1}{n}\sum_{i=1}^n \Bigg\{\mathbbm{1}(\|\bm{u}_{\tilde{k}}^{(i)}\| \leq \hat{C}E) \Big(2 - 
    \frac{4 \gamma^2}{\hat{C}^2 K^2 \rho^2}\Big)\Bigg(\frac{\hat{C} E}{\|\bm{u}_{\tilde{k}}^{(i)}\|}\Bigg)(f_i(\bm{w}_{\tilde{k}}) - f_i(\bm{w}^{*})) + \mathbbm{1}(\|\bm{u}_{\tilde{k}}^{(i)}\| > \hat{C}E) \Big(\frac{3 \hat{C} \|\bm{u}_{\tilde{k}}^{(i)}\|}{8 L E}\Big) \Bigg\}\Bigg]
    \\
    \leq \hat{C}\Big(\frac{L \|\bm{w}_{0} - \bm{w}^{*}\|^2}{\gamma} + \frac{\gamma}{L}\Big)\rho
    + \mathbb{E}\Bigg[\frac{1}{n}\sum_{i=1}^n {\mathbbm{1}\big(\|\bm{u}_{\tilde{k}}^{(i)}\| \leq \hat{C} E\big)} \Bigg\{\frac{\gamma \hat{C}}{L K \rho} + \Bigg(\frac{\hat{C} E}{\|\bm{u}_{\tilde{k}}^{(i)}\|}\Bigg)\frac{4 \gamma^2 \Delta_i^{*}}{\hat{C}^2 K^2 \rho^2}\Bigg\}\Bigg],
\end{multline*}
\normalsize
with $\tilde{k} \sim \textup{unif}[0,K-1]$. Further, this result holds for any $\bm{w}^{*} \in \arg \min_{\bm{w}' \in \mathbb{R}^d} f(\bm{w}')$.
\\
\\
Specifically, with $K = \big(\frac{2 \alpha \gamma}{\hat{C} E}\big) \frac{1}{\rho^2}$ and $E \leq \frac{\alpha}{2 \rho}$, where $\alpha \geq 1$ is another constant of our choice, \Cref{alg:dp-fedavg-2} has the following convergence guarantee:
\small
\begin{multline*}
        \mathbb{E}\Bigg[\frac{1}{n}\sum_{i=1}^n \Bigg\{\mathbbm{1}(\|\bm{u}_{\tilde{k}}^{(i)}\| \leq \hat{C}E) \Big(2 - \frac{\rho^2 E^2}{\alpha^2}\Big)\Bigg(\frac{\hat{C} E}{\|\bm{u}_{\tilde{k}}^{(i)}\|}\Bigg)(f_i(\bm{w}_{\tilde{k}}) - f_i(\bm{w}^{*})) + \mathbbm{1}(\|\bm{u}_{\tilde{k}}^{(i)}\| > \hat{C}E) \Big(\frac{3 \hat{C} \|\bm{u}_{\tilde{k}}^{(i)}\|}{8 L E}\Big) \Bigg\}\Bigg] 
        \\
        \leq
        \hat{C} \Big(\frac{L \|\bm{w}_{0} - \bm{w}^{*}\|^2}{\gamma} + \frac{\gamma}{L}\Big) \rho + 
        {\mathbb{E}\Bigg[\frac{1}{n}\sum_{i=1}^n \mathbbm{1}(\|\bm{u}_{\tilde{k}}^{(i)}\| \leq \hat{C} E) \Bigg\{\frac{\hat{C}^2}{2 \alpha L} + \Bigg(\frac{\hat{C} E}{\|\bm{u}_{\tilde{k}}^{(i)}\|}\Bigg) \frac{\Delta_i^{*} \rho E}{\alpha^2}\Bigg\}E\Bigg]}\rho.
\end{multline*}
\normalsize
\end{theorem}

\subsection{Proof of Theorem~\ref{thm-new-clip-cvx}}
\begin{proof}
Let us again set $\eta_k = \beta_k = \eta$, for all $k \geq 0$.
\\
\\
Everything remains the same till \cref{eq:pf-july15-2} in the proof of \Cref{thm-clip-cvx}, with $\text{clip}(.)$ replaced by $\text{norm}(.)$.
\small
\begin{multline}
    \label{eq:feb4-5}
    \mathbb{E}[\|\bm{w}_{k+1} - \bm{w}^{*}\|^2] \leq 
    \|\bm{w}_{k} - \bm{w}^{*}\|^2 + \frac{1}{n}\sum_{i = 1}^n \Big\{\underbrace{- 2 \eta\langle \text{norm}(\bm{u}_k^{(i)}, C), \bm{w}_{k} - \bm{w}^{*}\rangle + \eta^2 \big\|\text{norm}(\bm{u}_k^{(i)}, C)\big\|^2}_{A_i}\Big\}
    \\
    + \eta^2 \Big(\frac{q K d \log(1/\delta) C^2}{n^2 \varepsilon^2}\Big).
\end{multline}
\normalsize
Again, let us examine $A_i$ for each $i$. Also, as used in the proof of \Cref{thm-clip-cvx}, let $\hat{C} = \frac{C}{E}$.
\\
\\
\textbf{Case 1:} $\|\bm{u}_k^{(i)}\| > \hat{C} E$. Everything remains the same as Case 1 in the proof of \Cref{thm-clip-cvx}. Thus, 
\begin{equation}
    \label{eq:feb2-1}
    A_i \leq -\frac{3 \eta \hat{C}}{8 L} \|\bm{u}_k^{(i)}\|,
\end{equation}
for $\eta L \leq \frac{1}{2}$ and $\hat{C} \geq 4 \sqrt{L \max_{j \in [n]} \Delta_j^{*}}$.
\\
\\
\textbf{Case 2:} $\|\bm{u}_k^{(i)}\| \leq \hat{C} E$. Here:
\begin{equation}
    \label{eq:feb2-2}
    A_i \leq \Bigg(\frac{\hat{C} E}{\|\bm{u}_k^{(i)}\|}\Bigg) \Big({-2 \eta} \underbrace{\langle \bm{u}_k^{(i)}, \bm{w}_{k} - \bm{w}^{*} \rangle}_{B_i}\Big) + \eta^2 \hat{C}^2 E^2.
\end{equation}
For ease of notation henceforth, let us define:
\begin{equation}
    {z}_k^{(i)} := \Bigg(\frac{\hat{C} E}{\|\bm{u}_k^{(i)}\|}\Bigg).
\end{equation}
The bound for $B_i$ remains the same as the one in the proof of \Cref{thm-clip-cvx} (in \cref{eq:pf-july15-13}), i.e.,
\begin{equation}
    B_i \geq E (f_i(\bm{w}_{k}) - f_i(\bm{w}^{*})) - 2 \eta^2 L^2 E^3 (f_i(\bm{w}_{k}) - f_i^{*}),
\end{equation}
for $\eta L \leq 1$. Using this in \cref{eq:feb2-2}, we get:
\begin{flalign}
    \nonumber
    A_i & \leq {-2 \eta E} {z}_k^{(i)} \Big\{ (f_i(\bm{w}_{k}) - f_i(\bm{w}^{*})) - 2 \eta^2 L^2 E^2 (f_i(\bm{w}_k) - f_i^{*})\Big\} + \eta^2 \hat{C}^2 E^2
    \\
    \label{eq:feb2-3}
    & = {-2 \eta E} {z}_k^{(i)} \Big\{(f_i(\bm{w}_{k}) - f_i(\bm{w}^{*})){(1 - 2 \eta^2 L^2 E^2)} - 2 \eta^2 L^2 E^2 \Delta_i^{*}\Big\} + \eta^2 \hat{C}^2 E^2.
\end{flalign}
Combining the results of Case 1 and 2, i.e. \cref{eq:feb2-1} and \cref{eq:feb2-3}, we get:
\small
\begin{multline}
    \label{eq:feb4-4}
    A_i \leq {\eta E} \Bigg\{\mathbbm{1}(\|\bm{u}_k^{(i)}\| \leq \hat{C} E) \Big(4 \eta^2 L^2 E^2 \Delta_i^{*} {z}_k^{(i)} + \eta \hat{C}^2 E\Big) 
    \\
    - \mathbbm{1}(\|\bm{u}_k^{(i)}\| \leq \hat{C} E) (2 - 4 \eta^2 L^2 E^2) {z}_k^{(i)} (f_i(\bm{w}_{k}) - f_i(\bm{w}^{*}))
    - \mathbbm{1}(\|\bm{u}_k^{(i)}\| > \hat{C} E)\Big(\frac{3 \hat{C} \|\bm{u}_k^{(i)}\|}{8 L E}\Big)\Bigg\},
\end{multline}
\normalsize
for $\eta L \leq \frac{1}{2}$ and $\hat{C} \geq 4 \sqrt{L \max_{j \in [n]} \Delta_j^{*}}$.
\\
\\
Now using the above bound in \cref{eq:feb4-5}, plugging in ${z}_k^{(i)} = \frac{\hat{C} E}{\|\bm{u}_k^{(i)}\|}$, and following the same process and choice of $\eta = \frac{\gamma}{\hat{C} L E K} \frac{n \varepsilon}{\sqrt{q d \log(1/\delta)}}$ that we used in \Cref{thm-clip-cvx}, we get:
\small
\begin{multline}
    \label{eq:feb4-6}
    \mathbb{E}\Bigg[\frac{1}{n}\sum_{i=1}^n \Bigg\{\mathbbm{1}(\|\bm{u}_{\tilde{k}}^{(i)}\| \leq \hat{C}E) \Big(2 - 
    \frac{4 \gamma^2}{\hat{C}^2 K^2} \frac{n^2 \varepsilon^2}{q d \log(1/\delta)}\Big)\Bigg(\frac{\hat{C} E}{\|\bm{u}_{\tilde{k}}^{(i)}\|}\Bigg)(f_i(\bm{w}_{\tilde{k}}) - f_i(\bm{w}^{*})) 
    \\
    + \mathbbm{1}(\|\bm{u}_{\tilde{k}}^{(i)}\| > \hat{C}E) \Big(\frac{3 \hat{C} \|\bm{u}_{\tilde{k}}^{(i)}\|}{8 L E}\Big)\Bigg\}\Bigg]
    \leq \Big(\frac{L \|\bm{w}_{0} - \bm{w}^{*}\|^2}{\gamma} + \frac{\gamma}{L}\Big)\frac{\hat{C} \sqrt{q d \log(1/\delta)}}{n \varepsilon}
    \\
    + \mathbb{E}\Bigg[\frac{1}{n}\sum_{i=1}^n {\mathbbm{1}\big(\|\bm{u}_{\tilde{k}}^{(i)}\| \leq \hat{C} E\big)} \Bigg\{\frac{\gamma \hat{C}}{L K} \frac{n \varepsilon}{\sqrt{q d \log(1/\delta)}}
    + \frac{4 \gamma^2 \Delta_i^{*}}{\hat{C}^2 K^2} \frac{n^2 \varepsilon^2}{q d \log(1/\delta)} \Bigg(\frac{\hat{C} E}{\|\bm{u}_{\tilde{k}}^{(i)}\|}\Bigg)\Bigg\}\Bigg],
\end{multline}
\normalsize
with $\tilde{k} \sim \text{unif }[0,K-1]$ and $K > \frac{2 \gamma}{\hat{C} E} \frac{n \varepsilon}{\sqrt{q d \log(1/\delta)}}$ (so that $\eta L E \leq \frac{1}{2}$). Now setting $K = \frac{2 \gamma}{\hat{C} E} \Big(\frac{n \varepsilon}{\sqrt{q d \log(1/\delta)}}\Big)^2$ and $\rho = \frac{\sqrt{q d \log(1/\delta)}}{n \varepsilon}$ above gives us the final result.
\end{proof}

\section{Lemmas and some Facts used in the Proofs}
\begin{lemma}
\label{lem-mar15-1}
Suppose $f_i$ is convex and $L$-smooth over $\mathbb{R}^d$.
Let us set $\eta_k \leq \frac{1}{2L}$ for round $k$ of Algorithm \ref{alg:dp-fedavg} and \ref{alg:dp-fedavg-2}. Then:
\begin{equation*}
    \|\bm{w}_{k,E}^{(i)} - \bm{w}^{*}\|^2 \leq \|\bm{w}_{k} - \bm{w}^{*}\|^2 - \frac{\eta_k}{2L} \sum_{\tau=0}^{E-1}\|\nabla f_i(\bm{w}_{k,\tau}^{(i)})\|^2 + 2 \eta_k E  \Delta_i^{*},
\end{equation*}
where $\Delta_i^{*} := f_i(\bm{w}^{*}) - \min_{\bm{w}' \in \mathbb{R}^d} f_i(\bm{w}')$.
\end{lemma}
\begin{proof}
Let us define $f_i^{*} := \min_{\bm{w}' \in \mathbb{R}^d} f_i(\bm{w}')$. Then, $\Delta_i^{*} = f_i(\bm{w}^{*}) - f_i^{*}$.
\\
\\
For any $\tau \geq 0$, we have:
\begin{flalign}
    \nonumber
    \|\bm{w}_{k,\tau+1}^{(i)} - \bm{w}^{*}\|^2 & = \|\bm{w}_{k,\tau}^{(i)} - \bm{w}^{*}\|^2 - 2\eta_k \langle \nabla f_i(\bm{w}_{k,\tau}^{(i)}), \bm{w}_{k,\tau}^{(i)} - \bm{w}^{*} \rangle + \eta_k^2 \|\nabla f_i(\bm{w}_{k,\tau}^{(i)})\|^2
    \\
    \label{mar15-2}
    & \leq \|\bm{w}_{k,\tau}^{(i)} - \bm{w}^{*}\|^2 - 2\eta_k (f_i(\bm{w}_{k,\tau}^{(i)}) - f_i(\bm{w}^{*})) + \eta_k^2 \|\nabla f_i(\bm{w}_{k,\tau}^{(i)})\|^2
    \\
    & \leq \|\bm{w}_{k,\tau}^{(i)} - \bm{w}^{*}\|^2 - 2\eta_k (f_i(\bm{w}_{k,\tau}^{(i)}) - f_i^{*}) + 2 \eta_k \underbrace{(f_i(\bm{w}^{*}) - f_i^{*})}_{=\Delta_i^{*}} + \eta_k^2 \|\nabla f_i(\bm{w}_{k,\tau}^{(i)})\|^2
    \\
    \label{mar15-2-1}
    & \leq \|\bm{w}_{k,\tau}^{(i)} - \bm{w}^{*}\|^2 - \frac{\eta_k}{L} \|\nabla f_i(\bm{w}_{k,\tau}^{(i)})\|^2 + 2 \eta_k \Delta_i^{*} + 
    \eta_k^2 \|\nabla f_i(\bm{w}_{k,\tau}^{(i)})\|^2.
\end{flalign}
\Cref{mar15-2} follows by using the fact that each $f_i$ is convex. \Cref{mar15-2-1} follows using \Cref{fact1}.
\\
\\
Now if we set $\eta_k \leq \frac{1}{2L}$, then we get:
\begin{equation}
    \label{mar15-3}
    \|\bm{w}_{k,\tau+1}^{(i)} - \bm{w}^{*}\|^2 \leq \|\bm{w}_{k,\tau}^{(i)} - \bm{w}^{*}\|^2 - \frac{\eta_k}{2 L} \|\nabla f_i(\bm{w}_{k,\tau}^{(i)})\|^2 + 2 \eta_k \Delta_i^{*}.
\end{equation}
Doing this recursively for $\tau = 0$ through to $\tau=E-1$ and adding everything up gives us the desired result.
\end{proof}

\begin{lemma}
\label{lem-mar13-1}
Suppose each $f_i$ is $L$-smooth over $\mathbb{R}^d$ and $f_i^{*} := \min_{\bm{w}' \in \mathbb{R}^d} f_i(\bm{w}')$. Let us set $\eta_k \leq \frac{1}{L}$ for round $k$ of Algorithm \ref{alg:dp-fedavg} and \ref{alg:dp-fedavg-2}. Then:
\begin{equation*}
    \|\bm{w}_k - \bm{w}^{(i)}_{k, \tau}\|^2 \leq 2 \eta_k^2 L \tau^2 (f_i(\bm{w}_{k}) - f_i^{*}) \text{ } \forall \text{ } \tau \geq 1.
\end{equation*}
Thus, 
\[\big\|\bm{u}_k^{(i)}\big\|^2 \leq 2 L E^2 (f_i(\bm{w}_{k}) - f_i^{*}).\]
\end{lemma}
\begin{proof}
\begin{flalign}
    \label{eq:feb28-8}
    \|\bm{w}_k - \bm{w}^{(i)}_{k, \tau}\|^2 = \Big\|\eta_k \sum_{t=0}^{\tau-1} \nabla {f}_i(\bm{w}^{(i)}_{k, t})\Big\|^2
    & \leq \eta_k^2 \tau \sum_{t=0}^{\tau-1} \|\nabla {f}_i(\bm{w}^{(i)}_{k, t})\|^2,
\end{flalign} 
where the last step follows from \Cref{fact2}. Next, since $f_i$ is $L$-smooth, we have using \Cref{fact1}:  
\[\|\nabla {f}_i(\bm{w}^{(i)}_{k, t})\|^2 \leq 2L({f}_i(\bm{w}^{(i)}_{k, t}) - f_i^{*}).\]
Applying this in \cref{eq:feb28-8}, we get:
\begin{equation}
    \label{eq:feb26-22-1}
    \|\bm{w}_k - \bm{w}^{(i)}_{k, \tau}\|^2 \leq 2 \eta_k^2 L \tau \sum_{t=0}^{\tau-1} (f_i(\bm{w}^{(i)}_{k, t}) - f_i^{*}).
\end{equation}
But using the $L$-smoothness of $f_i$, we have for any $t \geq 1$:
\begin{flalign}
    f_i(\bm{w}_{k,t}^{(i)}) - f_i^{*} & = f_i(\bm{w}_{k,t-1}^{(i)} - \eta_k \nabla f_i(\bm{w}_{k,t-1}^{(i)})) - f_i^{*}
    \\
    & \leq (f_i(\bm{w}_{k,t-1}^{(i)}) - f_i^{*}) - \eta_k \|\nabla f_i(\bm{w}_{k,t-1}^{(i)})\|^2 + \frac{\eta_k^2 L}{2} \|\nabla f_i(\bm{w}_{k,t-1}^{(i)})\|^2
    \\
    & \leq (f_i(\bm{w}_{k,t-1}^{(i)}) - f_i^{*}) - \frac{\eta_k}{2} \|\nabla f_i(\bm{w}_{k,t-1}^{(i)})\|^2,
\end{flalign}
for $\eta_k L \leq 1$. Doing this recursively (and recalling that $\bm{w}_{k,0}^{(i)} = \bm{w}_k$), we get:
\begin{equation}
    f_i(\bm{w}_{k,t}^{(i)}) - f_i^{*} \leq (f_i(\bm{w}_{k}) - f_i^{*}) - \frac{\eta_k}{2}\sum_{t'=0}^{t-1} \|\nabla f_i(\bm{w}_{k,t'}^{(i)})\|^2 \leq f_i(\bm{w}_{k}) - f_i^{*}.
\end{equation}
Plugging this in \cref{eq:feb26-22-1}, we get:
\begin{equation}
    \|\bm{w}_k - \bm{w}^{(i)}_{k, \tau}\|^2 \leq 2 \eta_k^2 L \tau^2 (f_i(\bm{w}_{k}) - f_i^{*}).
\end{equation}
The upper bound on $\big\|\bm{u}_k^{(i)}\big\|^2$ follows by recalling that $\bm{u}_k^{(i)} = (\bm{w}_k - \bm{w}_{k,E}^{(i)})/\eta_k$.
\end{proof}

\begin{lemma}
\label{local_steps_lem}
Suppose each $f_i$ is $L$-smooth over $\mathbb{R}^d$. Then for both Algorithm \ref{alg:dp-fedavg} and \ref{alg:dp-fedavg-2}, we have:
\begin{equation*}
    \|\nabla f_i(\bm{w}_{k, \tau+1}^{(i)})\|^2 \leq \|\nabla f_i(\bm{w}_{k, \tau}^{(i)})\|^2 - \Big(\frac{2}{\eta_k L} - 1\Big)\|\nabla f_i(\bm{w}_{k,\tau+1}^{(i)}) - \nabla f_i(\bm{w}_{k,\tau}^{(i)}) \|^2,
\end{equation*}
for any $i \in [n]$, $k \in \{0,\ldots,K-1\}$ and $\tau \in \{0,\ldots,E-1\}$.
\end{lemma}
\begin{proof}
Since each $f_i$ is $L$-smooth, we have by using the co-coercivity of the gradient:
\begin{flalign}
    & \langle \nabla f_i(\bm{w}_{k, \tau+1}^{(i)}) - \nabla f_i(\bm{w}_{k, \tau}^{(i)}), \bm{w}_{k, \tau+1}^{(i)} - \bm{w}_{k, \tau}^{(i)} \rangle \geq \frac{1}{L}\|\nabla f_i(\bm{w}_{k, \tau+1}^{(i)}) - \nabla f_i(\bm{w}_{k, \tau}^{(i)})\|^2.
\end{flalign}
Now using the fact that $\bm{w}_{k, \tau+1}^{(i)} - \bm{w}_{k, \tau}^{(i)} = -\eta_k \nabla f_i(\bm{w}_{k, \tau}^{(i)})$ above, 
we get:
\begin{multline}
    L \langle \nabla f_i(\bm{w}_{k, \tau+1}^{(i)}) - \nabla f_i(\bm{w}_{k, \tau}^{(i)}), -\eta_k \nabla f_i(\bm{w}_{k, \tau}^{(i)}) \rangle \geq \|\nabla f_i(\bm{w}_{k, \tau+1}^{(i)})\|^2 + \|\nabla f_i(\bm{w}_{k, \tau}^{(i)})\|^2 
    \\
    - 2 \langle \nabla f_i(\bm{w}_{k, \tau+1}^{(i)}), \nabla f_i(\bm{w}_{k, \tau}^{(i)}) \rangle.
\end{multline}
Rearranging the above a bit, we get:
\begin{equation}
    \label{eq:mar12-1000}
    (2 - \eta_k L) \langle \nabla f_i(\bm{w}_{k, \tau+1}^{(i)}), \nabla f_i(\bm{w}_{k, \tau}^{(i)}) \rangle \geq \|\nabla f_i(\bm{w}_{k, \tau+1}^{(i)})\|^2 + (1 - \eta_k L) \|\nabla f_i(\bm{w}_{k, \tau}^{(i)})\|^2.
\end{equation}
But, we also have:
\begin{equation}
    \langle \nabla f_i(\bm{w}_{k, \tau+1}^{(i)}), \nabla f_i(\bm{w}_{k, \tau}^{(i)}) \rangle = \frac{1}{2}\Big(\|\nabla f_i(\bm{w}_{k, \tau+1}^{(i)})\|^2 + \|\nabla f_i(\bm{w}_{k, \tau}^{(i)})\|^2 - \|\nabla f_i(\bm{w}_{k,\tau+1}^{(i)}) - \nabla f_i(\bm{w}_{k,\tau}^{(i)}) \|^2\Big).
\end{equation}
Using this in \cref{eq:mar12-1000} and simplifying a bit, we get:
\begin{equation}
    \label{eq:mar12-1001}
    \|\nabla f_i(\bm{w}_{k, \tau+1}^{(i)})\|^2 \leq \|\nabla f_i(\bm{w}_{k, \tau}^{(i)})\|^2 - \Big(\frac{2}{\eta_k L} - 1\Big)\|\nabla f_i(\bm{w}_{k,\tau+1}^{(i)}) - \nabla f_i(\bm{w}_{k,\tau}^{(i)}) \|^2.
\end{equation}
This completes the proof.
\end{proof}

\begin{lemma}
\label{lem-apr13-1}
Suppose each $f_i$ is $L$-smooth over $\mathbb{R}^d$. Let us set $\eta_k \leq \frac{1}{2 L E}$ for round $k$ of Algorithm \ref{alg:dp-fedavg} and \ref{alg:dp-fedavg-2}. Then:
\begin{equation*}
    \|\bm{w}_k - \bm{w}^{(i)}_{k, \tau}\| \leq 2 \eta_k \tau \|\nabla f_i(\bm{w}_k)\| \text{ } \forall \text{ } \tau \geq 1.
\end{equation*}
\end{lemma}
\noindent The reader might be wondering that \Cref{lem-mar13-1} also bounds $\|\bm{w}_k - \bm{w}^{(i)}_{k, \tau}\|$, so why do we need this lemma? The difference is that this lemma provides a stronger bound at the cost of a stronger requirement on $\eta_k$, whereas \Cref{lem-mar13-1} provides a weaker bound but it imposes a weaker requirement on $\eta_k$. This lemma is used only in the proof of \Cref{cor-1}, while \Cref{lem-mar13-1} is used in the proofs of Theorems \ref{thm-clip-cvx-short} and \ref{thm-new-clip-cvx-short}.
\begin{proof}
\begin{flalign}
    \label{eq:feb28-8-2}
    \|\bm{w}_k - \bm{w}^{(i)}_{k, \tau}\| = \Big\|\eta_k \sum_{t=0}^{\tau-1} \nabla {f}_i(\bm{w}^{(i)}_{k, t})\Big\|
    & \leq \eta_k \sum_{t=0}^{\tau-1} \|\nabla {f}_i(\bm{w}^{(i)}_{k, t})\|.
\end{flalign} 
But:
\begin{flalign}
    \nonumber
    \|\nabla {f}_i(\bm{w}^{(i)}_{k, t})\| & = \|\nabla {f}_i(\bm{w}^{(i)}_{k, t}) - \nabla {f}_i(\bm{w}_k) + \nabla {f}_i(\bm{w}_k)\|
    \\
    \nonumber
    & \leq \|\nabla {f}_i(\bm{w}_k)\| + \|\nabla {f}_i(\bm{w}^{(i)}_{k, t}) - \nabla {f}_i(\bm{w}_k)\|
    \\
    \label{eq:feb28-7}
    & \leq \|\nabla {f}_i(\bm{w}_k)\| + L\|\bm{w}^{(i)}_{k, t} - \bm{w}_k\|.
\end{flalign}
Putting \cref{eq:feb28-7} back in \cref{eq:feb28-8-2}, we get:
\begin{flalign}
    \label{eq:feb28-new-1}
    \|\bm{w}_k - \bm{w}^{(i)}_{k, \tau}\| \leq \eta_k \tau \|\nabla {f}_i(\bm{w}_k)\| + \eta_k L \sum_{t=0}^{\tau-1}\|\bm{w}^{(i)}_{k, t} - \bm{w}_k\|.
\end{flalign}
We claim that $\|\bm{w}_k - \bm{w}^{(i)}_{k, \tau}\| \leq 2 \eta_k \tau \|\nabla f_i(\bm{w}_k)\|$ for $\eta_k L E \leq 1/2$. We shall prove this by induction. Let us first check the base case of $\tau=1$. Observe that:
\[\|\bm{w}_k - \bm{w}^{(i)}_{k, 1}\| = \eta_k \|\nabla f_i(\bm{w}_k)\| \leq 2\eta_k\|\nabla f_i(\bm{w}_k)\|.\]
Hence, the base case is true.
Assume the hypothesis holds for $t \in \{0,\ldots,\tau-1\}$. Let us now put our induction hypothesis into \cref{eq:feb28-new-1} to see if the hypothesis is true for $\tau$ as well.
\begin{flalign*}
    \|\bm{w}_k - \bm{w}^{(i)}_{k, \tau}\| & \leq \eta_k \tau \|\nabla {f}_i(\bm{w}_k)\| + \eta_k L \sum_{t=0}^{\tau-1} 2\eta_k t \|\nabla f_i(\bm{w}_k)\|
    \\
    & \leq \eta_k \tau \|\nabla {f}_i(\bm{w}_k)\| + (\eta_k L) \eta_k \tau^2 \|\nabla {f}_i(\bm{w}_k)\|
    \\
    & \leq \eta_k \tau \|\nabla {f}_i(\bm{w}_k)\| + \eta_k \tau (\eta_k L \tau) \|\nabla {f}_i(\bm{w}_k)\|
    \\
    & \leq \eta_k \tau \|\nabla {f}_i(\bm{w}_k)\| + 0.5 \eta_k \tau \|\nabla {f}_i(\bm{w}_k)\| < 2 \eta_k \tau \|\nabla {f}_i(\bm{w}_k)\|.
\end{flalign*}
The second last inequality is true because $\eta_k L \tau \leq \eta_k L E \leq \frac{1}{2}$, per our choice of $\eta_k$.
\\
\\
Thus, the hypothesis holds for $\tau$ as well. So by induction, our claim is true.
\end{proof}

\begin{fact}[\textbf{\cite{nesterov2018lectures}}]
\label{fact1}
For an $L$-smooth function $h: \mathbb{R}^d \xrightarrow{} \mathbb{R}$ with $h^{*} = \min_{\bm{x} \in \mathbb{R}^d} h(\bm{x})$ and $L>0$, $\|\nabla h(\bm{x})\|^2 \leq 2L(h(\bm{x}) - h^{*})$.
\end{fact}

\begin{fact}
\label{fact2}
For any $p > 1$ vectors $\{\bm{y}_1,\ldots,\bm{y}_p\}$, $\|\sum_{i=1}^p \bm{y}_i\|^2 \leq p \sum_{i=1}^p \|\bm{y}_i\|^2$.
\end{fact}
\noindent \Cref{fact2} follows from Jensen's inequality.

\begin{fact}
\label{new-fact-1}
Suppose $x \in (0,1)$. Then for any positive integer $m$ such that $m x \leq \frac{1}{4}$, we have:
\begin{equation}
    \frac{1 - (1-x)^m}{x} \leq m\Big(1 - \frac{11(m-1)}{24}x\Big).
\end{equation}
\end{fact}

\begin{proof}
Using the Binomial expansion, we have:
\begin{equation}
    (1-x)^m \geq 1 - m x + \frac{m(m-1)}{2}x^2 - \frac{m(m-1)(m-2)}{6}x^3.
\end{equation}
Thus, 
\begin{flalign}
    \frac{1 - (1-x)^m}{x} & \leq m\Big\{1 - \frac{(m-1)}{2}x + \frac{(m-1)(m-2)}{6}x^2\Big\}
    \\
    \label{eq:mar7-100}
    & \leq m\Big\{1 - \frac{(m-1)}{2}x + \frac{(m-1)}{6}x\underbrace{\Big(\frac{m-2}{4 m}\Big)}_{\leq \frac{1}{4}}\Big\}
    \\
    & \leq m\Big(1 - \frac{11(m-1)}{24}x\Big).
\end{flalign}
Here, \cref{eq:mar7-100} follows from the fact that $m x \leq \frac{1}{4}$. 
\end{proof}

\section{Experimental Details}
\label{sec:expt_det}
First, we explain the procedure we have used to generate heterogeneous data for our FL experiments in \Cref{sec:expts}. For each dataset (individually), the training data was first sorted based on labels and then divided into $5n$ equal data-shards, where $n$ is the number of clients. Splitting the data in this way ensures that each shard contains data from only one class for all datasets (and because $n$ was chosen appropriately). Now, each client is assigned 5 shards chosen uniformly at random without replacement which ensures that each client can have data belonging to at most 5 distinct classes.
\\
\\
Next, we specify the learning rate schedule for our 
experiments in \Cref{sec:expts}.
We use $\beta_k = \eta_k$ for all $k$. We employ the learning rate scheme suggested in \cite{bottou2012stochastic} where we decrease the local learning rate by a factor of 0.99 after every round, i.e. $\eta_k = (0.99)^k \eta_0$. 
We search the best initial local learning rates $\eta_0$ over $\{10^{-3}, 2 \times 10^{-3}, 4 \times 10^{-3}, 8 \times 10^{-3}, 1.6 \times 10^{-2}, 3.2 \times 10^{-2}, 6.4 \times 10^{-2}\}$ in each case. Server momentum = 0.8 is also applied (at the server).

\end{document}

%% file: defpack2.tex
\usepackage[utf8]{inputenc} 
\usepackage[T1]{fontenc}    
\usepackage{hyperref}       
\usepackage{nicefrac}       
\usepackage{microtype}      
\usepackage{xcolor}
\usepackage{framed}
\colorlet{shadecolor}{pink}
\usepackage{authblk}

\usepackage{graphicx}
\usepackage{soul}
\usepackage{subcaption}
\usepackage{booktabs} 
\usepackage{tablefootnote}

\usepackage{amsmath,amsthm,amssymb,amsfonts}
\usepackage{algorithm}
\usepackage{algorithmic}
\usepackage{enumerate}
\usepackage{cleveref}
\usepackage{comment}
\usepackage{bm}
\usepackage{pifont}

\theoremstyle{plain}
\newtheorem{theorem}{Theorem}[]

\newtheorem{lemma}[]{Lemma}
\newtheorem{proposition}{Proposition}

\newtheorem{definition}{Definition}[]
\newtheorem{assumption}{Assumption}[]
\newtheorem{fact}{Fact}

\newtheorem{remark}{Remark}[]

\usepackage[colorinlistoftodos,prependcaption,disable]{todonotes}
\newcommand{\rudrajit}[1]{\todo[color=red!25, inline]{Rudrajit: #1}}